\documentclass{article} % For LaTeX2e
\usepackage{iclr2027_conference,times}

\usepackage{amsmath,amsfonts,bm}

\def\eqref#1{equation~\ref{#1}}
\def\1{\bm{1}}

\DeclareMathAlphabet{\mathsfit}{\encodingdefault}{\sfdefault}{m}{sl}
\SetMathAlphabet{\mathsfit}{bold}{\encodingdefault}{\sfdefault}{bx}{n}

\iclrfinalcopy

\usepackage{hyperref}
\usepackage{url}
\usepackage{amsthm}
\newtheorem{proposition}{Proposition}
\newtheorem{theorem}{Theorem}
\newtheorem{lemma}{Lemma}
\newtheorem{remark}{Remark}

\newtheorem{definition}{Definition}
\usepackage{booktabs}
\usepackage{multirow}
\usepackage{graphicx}
\usepackage{amsmath}
\usepackage{amsthm}
\usepackage{enumitem}
\usepackage{tabularx}
\usepackage{amssymb}
\usepackage{physics}
\usepackage{nicefrac}
\newcommand{\pair}[2]{\begin{tabular}[t]{@{}c@{}}#1\\#2\end{tabular}}

\usepackage{tikz}
\usetikzlibrary{arrows.meta,positioning,calc}

 \newcommand{\dg}{\bar\nabla} \newcommand{\rmech}{r_{\mathrm{mech}}}
\newcommand{\Pimodel}{\Pi_{\mathrm{model}}} \newcommand{\Pistar}{\Pi^\star}

\hypersetup{hidelinks}

\title{Poisson-GENERIC Neural Operators: Exact Metriplectic Structure in Function Space via Casimir Entropies}

\author{Jason Sulskis\thanks{Also with the Georgia Tech Research Institute, Atlanta, GA, USA.} \\
Department of Computer Science \\
University of Illinois at Chicago \\
Chicago, IL 60607, USA \\
\texttt{jsulsk1@uic.edu} \\
\And
Sathya Ravi \\
Department of Computer Science \\
University of Illinois at Chicago \\
Chicago, IL 60607, USA \\
\texttt{sathya@uic.edu} \\
}

\usepackage{algorithm}
\usepackage{algpseudocode}
\usepackage{xcolor}

\begin{document}
\maketitle
\lhead{Preprint. Under review.}   % overrides "Published as a conference paper at ICLR 2027"

\begin{abstract}
Thermodynamically consistent neural operators thus far enforce the GENERIC degeneracy conditions by projecting the reversible
operator, which makes it depend on the state and does not guarantee the Jacobi identity. We instead augment the observed
field $u$ with a latent entropy density $s$, so that $S=\int s$ is a Casimir and $L\,\delta S/\delta z=0$ holds with no
projection. The reversible operator is a constant Poisson multiplier for the linear equations and the compatible
Lie--Poisson pencil $L=\alpha D+\lambda(uD+Du)$ for nonlinear transport. The energy is a fixed mechanical quadratic plus a
learned gauge-free potential and a convex internal energy. The friction operator $M=AA^\top$ satisfies
$M\,\delta E/\delta z=0$ pointwise; it can represent diffusion and damping but cannot act as a second transport generator.
Skewness, positivity and both degeneracies hold for any parameter values. We show that the continuum pencil satisfies Jacobi
(verified on the alias-free band), that the closed-form heat and damped-wave friction operators are in the family, that the
physical energy is bounded under coercivity assumptions, and that a discrete-gradient integrator keeps both laws in discrete
time. On four PDEs in 1D and 2D with three backbones (FNO, Transolver, CNO), the model wins 63/72 seed-level comparisons
against unconstrained baselines at equal optimizer steps, recovers the exact transport and wave symbols, matches heat and
Burgers dissipation within 10\%, and dissipates nothing on advection. At equal wall-clock the unconstrained FNO is the more
accurate ten-frame surrogate, but the structure pays off in the energy balance at any budget. Over a hundred frames, in and
out of distribution, the model never raises the physical energy on a stable rollout and conserves it where dissipation
vanishes, whereas the unconstrained FNO loses a quarter of a conserved energy off distribution and projected GENERIC-FNO
creates energy on most frames. We report two limits: late-time under-damping on heat and divergence of the explicit
integrator on the wave.
\end{abstract}

\newcommand{\TODO}[1]{\textcolor{red}{[#1]}}

\renewcommand{\ip}[2]{\langle #1, #2\rangle}

\section{Introduction}
\label{sec:intro}

\paragraph{Thermodynamic structure in neural operators.}
Neural operators can quickly approximate solutions to time-dependent PDEs on multiple grid resolutions \citep{kovachki2023neural,li2021fno}. Here we ask what a neural operator can guarantee when it feeds its own predictions back into itself. A model can predict one frame accurately and still make a damped wave grow over many frames. \emph{GENERIC} (General Equation for Non-Equilibrium Reversible--Irreversible Coupling) and its equivalent \emph{metriplectic} formulation address this by putting energy conservation and entropy production into the equation that the model uses to evolve the state \citep{grmela1997dynamics,ottinger1997dynamics,kaufman1984dissipative,morrison1986paradigm}. That is, it naturally provides a way to enforce thermodynamic laws by construction within heat conduction, damped waves, viscous fluids, and kinetic equations \citep{ottinger2005beyond}. For a state $z$ with energy $E[z]$ and entropy $S[z]$,
\begin{equation}
\partial_t z
=
L\,\frac{\delta E}{\delta z}
+
M\,\frac{\delta S}{\delta z},
\qquad
L\,\frac{\delta S}{\delta z}=0,
\qquad
M\,\frac{\delta E}{\delta z}=0,
\label{eq:generic}
\end{equation}
where $L$ is a skew-adjoint Poisson operator, so it satisfies the Jacobi identity, and $M$ is symmetric positive semi-definite. The two degeneracy conditions give
\begin{equation}
\dot E=0,\qquad
\dot S
=
\left\langle
\frac{\delta S}{\delta z},
M\,\frac{\delta S}{\delta z}
\right\rangle
\ge 0,
\label{eq:degeneracy}
\end{equation}
so the reversible part obeys the first law and the irreversible part obeys the second law. 

\paragraph{There are three reasons to learn in GENERIC form.} First, the skew-adjointness of $L$, the positive
semidefiniteness of $M$ and both degeneracy conditions hold for all parameter values, so that the first and second laws hold
before and after training and for every input. No loss term or penalty weight is required, and the guarantees are the same
out of distribution. Second, the formulation separates the physics, such as the energy balance laws and the reversible
operator, from the system-dependent functions that must be learned, such as the internal energy, the potential and the
friction coefficients. Third, the model can be interpreted physically: its operator symbols and dissipation rate can be
compared directly with their exact counterparts. Finite-dimensional metriplectic networks apply this structure to learned
dynamical systems
\citep{hernandez2021structure,lee2021machine,zhang2022gfinns,gruber2023reversible,gruber2025efficiently}. GENERIC-FNO extends
it to neural operators, with $E$ and $S$ learned neural functionals and $L$ and $M$ Fourier multipliers
\citep{sulskis2026generic}. However, it enforces the degeneracy condition by projection, $L=(I-P_S)D_L(I-P_S)$. This
guarantees $L\,\delta S/\delta z=0$, but the projected operator depends on the state and in general violates the Jacobi
identity: the residuals in Table~\ref{tab:structure-full} range from $10^{-3}$ to $1$. Moreover, a model that conserves some
learned energy need not conserve the physical energy. On the wave equation, GENERIC-FNO's physical energy reaches $34\times$
the solver's (Table~\ref{tab:relevance}).

\paragraph{The gap and how we address it.} We augment the state with a latent entropy density $s$, set $S=\int s$, and let the reversible operator act only on the observed field. Entropy is then a \emph{Casimir} functional of $L$, that is, its variational derivative lies in the null space of $L$ at every state. Hence, $L\,\delta S/\delta z=0$ holds by the parameterization and does not need to be learned through a loss penalty, and
the reversible dynamics cannot change the entropy, and so no projection is needed. For the linear equations $L$ is a constant Fourier multiplier, while for the nonlinear transport it is the compatible Poisson pencil \(\alpha D+\lambda(uD+Du),\) from an infinite family of related nonlinear evolution equations generated from the Korteweg–de Vries (KdV) equation \citep{magri1978simple}. The model splits its energy into two parts: the physical mechanical energy (such as kinetic, elastic, or potential energy) plus learned pointwise constitutive terms that represent local material properties or relationships that are not known in advance. The dissipation operator uses the Onsager factorization $M=AA^{\!\top}$, which cannot decrease entropy since it guarantees that $M$ is symmetric positive semi-definite.

\paragraph{Our contributions.}
\begin{enumerate}
\item We construct a metriplectic neural operator on function space with a
Poisson operator \(L\), a latent entropy density whose integral is a Casimir,
and an Onsager friction operator. Skewness, symmetry, positivity, degeneracy,
and the discrete Jacobi identity on alias-free states hold by construction,
giving the corresponding energy and entropy laws
(Propositions~\ref{prop:pencil}--\ref{prop:structure}).

\item We prove that the parameterization contains the exact friction operators
for heat conduction and the damped wave while satisfying
\(M\,\delta E/\delta z=0\) without a global projection
(Proposition~\ref{prop:repr}). Under a verifiable curvature condition on the
learned constitutive energy, we also bound the physical mechanical energy
along the continuous flow and exactly solved discrete-gradient steps
(Theorem~\ref{thm:bound}).

\item We develop a discrete-gradient integrator that exactly preserves the
discrete energy and entropy laws, including for the state-dependent Poisson
pencil. For linear reversible dynamics, we give an exact unitary integrator
with no numerical growth or decay
(Theorem~\ref{thm:dg}, \S\ref{sec:integration}).

\item We test the method on four equations in one and two dimensions with
three neural-operator backbones. We compare against GENERIC-FNO and
compute-matched unconstrained baselines on long-horizon, data-efficiency, and
out-of-distribution tests, and report dissipation and ablation studies
(\S\ref{sec:experiments}).
\end{enumerate}

\section{Related Work}
\label{sec:related}

Architectures for operator learning include Fourier neural operators (FNOs) \citep{li2021fno}, DeepONets
\citep{lu2021deeponet}, wavelet and Laplace operators \citep{tripura2023wavelet,cao2024lno}, alias-free convolutional
operators \citep{raonic2023convolutional} and transformers for irregular geometries
\citep{wu2024transolver,luo2025transolverpp,alkin2024upt}. None of them constrains conservation or dissipation. Following
\citet{sulskis2026generic} we use FNO, Transolver and CNO-style backbones, to separate the effect of the architecture from
that of the thermodynamic structure; the backbone encodes the latent entropy and parameterizes the friction coefficients,
and as a residual map it is also the unconstrained baseline. In finite dimensions such structure has been built into
Hamiltonian and Lagrangian networks, symplectic integrators and metriplectic methods
\citep{greydanus2019hamiltonian,cranmer2020lagrangian,jin2020sympnets,hernandez2021structure,yu2021onsagernet,lee2021machine,zhang2022gfinns,gruber2023reversible,gruber2025efficiently}.
Metriplectic models enforce skewness of $L$, positivity of $M$ and the degeneracy conditions, but the Jacobi identity is
generally lost once $L$ depends on the state \citep{sulskis2026generic}. Poisson neural networks obtain a state-dependent
bracket by pulling a constant one back through a learned symplectomorphism \citep{jin2022poisson}; our constant operator is
the simplest instance, and \S\ref{app:pullback} gives a function-space version. For nonlinear transport we use a compatible
Poisson pencil, in which every linear combination of two Hamiltonian operators is again Poisson, as in the KdV hierarchy
\citep{magri1978simple,olver1993applications}. To our knowledge pencils have not been used before to learn Poisson brackets.

The thermodynamic ingredients are taken from GENERIC field theory. There the entropy is a Casimir, and an internal energy or
entropy density accompanies the mechanical variables
\citep{grmela1997dynamics,ottinger2005beyond,kaufman1984dissipative,morrison1984bracket,morrison1998hamiltonian,marsden1999introduction};
the friction operators often have the factored form $CD'C^{\!\top}$, whose state dependence enforces the energy degeneracy
in line with Onsager reciprocity \citep[Ch.~2]{ottinger2005beyond,onsager1931reciprocal,degroot1962nonequilibrium}. Our model
takes over a latent entropy density, a partly physical, partly learned energy, and a friction parameterization that contains
the heat and damped-wave operators exactly, with the parity of its derivative orders restricted so that it can diffuse and
damp but cannot transport (both the latent entropy and this restriction are ablated in \S\ref{sec:experiments}). In discrete
time, discrete-gradient methods preserve first integrals at any step size
\citep{gonzalez1996time,mclachlan1999geometric,celledoni2012preserving}, but a metriplectic scheme must also respect the
degeneracy conditions \citep{romero2009thermodynamically,ottinger2018generic,gruber2025efficiently}, which is why GENERIC-FNO
fell back on explicit Euler: higher-order methods broke its exact degeneracy. Here the entropy is linear, so its discrete
gradient is exact and lies in the kernel of $L$, and the Gonzalez method then conserves the discrete energy and produces
entropy even for the state-dependent Poisson pencil (Theorem~\ref{thm:dg}). The linear reversible part is treated by an exact
unitary map on each Fourier mode inside a Strang splitting \citep{strang1968construction}.

\newcommand{\PiK}{\Pi_{K}}

\section{Method: Poisson-GENERIC Neural Operators}
\label{sec:method}
\subsection{Learning problem and notation}\label{sec:problem}
\paragraph{The data.} We take a time series of PDE solutions, snapshot trajectories $u_0,u_1,\dots,u_T$ on the periodic domain $\Omega=[0,2\pi)^d$, $d\in\{1,2\}$, sampled on an
$N$-point grid at a fixed spacing $\Delta t$. The observed field at each time is $u_t:\Omega\to\mathbb R^p$, with $p=1$
for the scalar equations and $p=2$ for the damped wave, where both position and velocity are observed. Our task is to learn a \emph{one-frame map} $\Phi_\theta$ that takes one snapshot to the next, $u_t \mapsto u_{t+1}$, and is applied autoregressively. The unconstrained baseline, where nothing forces it to conserve energy or produce entropy, is a residual map $\Phi_\theta(u)=u+G_\theta(u)$, with $G_\theta$ being a "black-box" network (FNO, Transolver, or CNO) predicting the increment between frames, with no physical structure imposed. Our model does not predict the increment directly, rather, it learns the operators and constitutive functions of
equation~\ref{eq:generic}. Specifically the Poisson and friction operators and the energy and entropy functionals. It then obtains
$u_{t+1}$ by integrating that equation numerically over one \emph{frame}, the interval $\Delta t$ between two snapshots (we also call $u_t$ frame $t$). Thus the network parameterizes the
dynamics and the one-frame map, $\Phi_\theta(u_t)$, is the time integrator applied to those dynamics. Throughout, a \emph{step} is one application of an integrator, which may cover a frame or a fraction of one, and an \emph{optimizer step} is one parameter update in training. The structural guarantees follow from the
form of the equation and do not depend on the training data.  

\paragraph{Notation.}
The spectral derivative along axis $a$ is
$\partial_a f=\mathcal F^{-1}[ik_a\hat f]$, where $k_a$ is the
$a$th component of the wavevector $k$. The maximum component magnitude is
$|k|_\infty=\max_a |k_a|$. The band $|k|_\infty\le m$ is an interval of
modes in one dimension and a square in two. A \emph{Fourier multiplier}
with symbol $\sigma(k)$ acts as
$f\mapsto\mathcal F^{-1}[\sigma(k)\hat f]$. It acts on band $m$ if
$\sigma(k)=0$ for $|k|_\infty>m$. Let $n=N^{1/d}$ be the number of grid points along each axis. The
Nyquist wavenumber is $n/2$, and the $2/3$-rule de-aliasing projector $P$
retains modes with $|k|_\infty<n/3$. The product of two fields in the
range of $P$ has bandwidth less than $2n/3$. Any unresolved mode that
folds back onto the grid has at least one component with magnitude
greater than $n/3$, so applying $P$ to the product removes the aliased
contributions. A product is fully resolved on the grid if the
bandwidths of its factors sum to less than $n/2$. A state and a set
of test functionals are alias-free for the Jacobi identity when every
cubic product in the identity is fully resolved. We use the normalized inner product
$\langle f,g\rangle=|\Omega|^{-1}\int_\Omega f\cdot g\,\dd x$ and
write $\|\cdot\|$ for its induced norm. On the grid, the inner product
sums over channels and grid points and divides by $N$. The functional
gradient $\delta F/\delta z$ is taken with respect to this inner
product. On the grid, it equals $N$ times the Euclidean gradient
returned by automatic differentiation with respect to the grid values
(Appendix~\ref{app:innerproduct}).\footnote{Averaging over both channels
and grid points introduces a relative factor of $(p+1)/p$ between the
normalizations of the mechanical- and total-energy gradients. With
this scaling, the exact linear step in \S\ref{sec:integration}
integrates only part of the intended linear dynamics, and the learned
reversible symbols converge to a fraction of their exact values. We
identified this discrepancy by comparing the learned and exact symbols
(\S\ref{sec:diagnostics}).}

\subsection{The model}\label{sec:model}
\begin{definition}[State]\label{def:state}
$z=(u,s)\in\mathcal H=L^2(\Omega;\mathbb R^{p+1})$, where $u$ is the observed field and $s:\Omega\to\mathbb R$ is a latent
\emph{entropy density}. The entropy density is not observed: at the start of each rollout it is initialized by an encoder from
the backbone family of the baseline, $s_0=\psi_\eta(u_0)$, and afterwards it is advanced by the dynamics and receives no direct
supervision.
\end{definition}
\begin{definition}[Energy and entropy]\label{def:energy}
The energy is a fixed mechanical part plus two learned pointwise constitutive terms,
\begin{equation}
E_\theta[z]=Q(u)+\big\langle1,\varphi_\theta(u)\big\rangle+\big\langle1,e_\theta(s)\big\rangle,
\qquad
Q(u)=\begin{cases}\tfrac12\|u\|^2,&\text{scalar fields},\\[1pt]
\tfrac12\big(\|v\|^2+c^2\|\nabla q\|^2\big),&\text{wave, }u=(q,v),\end{cases}
\label{eq:energy}
\end{equation}
where $\varphi_\theta:\mathbb R^p\to\mathbb R$ is applied to $u(x)$ at every point and $e_\theta:\mathbb R\to\mathbb R$ to $s(x)$. For the
wave, $Q$ is the mechanical energy; for the scalar equations it is the quadratic that advection conserves and heat conduction
and Burgers dissipate. It is fixed, not learned. The \emph{internal energy}
$e_\theta(s)=T_0s+\sum_iw_i\operatorname{softplus}(a_is+b_i)$ with $w_i,a_i>0$ is convex and increasing, so the \emph{temperature}
$T(s):=e_\theta'(s)$ satisfies $T\ge T_0>0$; the floor is needed because the friction operator contains $1/T$. The \emph{potential} is
\begin{equation}
\varphi_\theta=\psi_\theta-\Pi_{\le2}\psi_\theta,\qquad
\psi_\theta=\operatorname{softplus}\big(\operatorname{MLP}_\theta(u)\big)\ge0,
\label{eq:potential}
\end{equation}
where $\Pi_{\le2}$ is the orthogonal projection, in $L^2(\gamma)$ for a fixed Gaussian measure $\gamma$ on the range of the training
data, onto polynomials of degree at most two in the $p$ components of $u$. The entropy is fixed, not learned:
$S[z]=\langle1,s\rangle$, so that $\delta S/\delta z=(0,1)$.
\end{definition}
\paragraph{Why project the potential.}

In $\langle 1,\varphi_\theta(u)\rangle$, a term $\frac{a}{2}|u|^2$
can be absorbed into $Q$ as $Q+aI$. This \emph{quadratic gauge} appeared
in the unconstrained potential (Remark~\ref{rem:gauge}). When the added
term rescales $Q$ to $(1+\kappa)Q$, the learned reversible operator
compensates by $(1+\kappa)^{-1}$, leaving the exact linear step to advance
that fraction of the linear flow and the explicit update to carry the
rest. We fix the gauge by setting
$\varphi_\theta=\psi_\theta-\Pi_{\le2}\psi_\theta$.
With a Taylor projection through degree two, the potential contributes
no linear term to the energy gradient about the reference state, so
the exact linear step handles the full linear reversible flow there
(\S\ref{sec:integration}). Since $\psi_\theta\ge0$, we obtain the
quadratic lower bound $\varphi_\theta\ge-\Pi_{\le2}\psi_\theta$ used in
Theorem~\ref{thm:bound}.

\begin{definition}[Poisson operator]\label{def:poisson}
$L=\operatorname{diag}(L_u,0)$ acts only on the observed field. For scalar fields,
\begin{equation}
L_u=\alpha D+\lambda\,P(uD+Du)P,\qquad D=\sum_{a=1}^d\partial_a,
\label{eq:pencil}
\end{equation}
with learned scalars $\alpha,\lambda$: $\alpha D$ is the constant Fourier multiplier with symbol $i\alpha\sum_ak_a$ on the band
$1\le|k|_\infty\le m_L$, and $uD+Du$ is the Lie--Poisson operator $g\mapsto uDg+D(ug)$. For the wave,
$L_u=J\operatorname{diag}(D_d,D_d)$ with $J=\big(\begin{smallmatrix}0&I\\-I&0\end{smallmatrix}\big)$ and $D_d$ a learned even
multiplier $d(k)\ge0$ on the band.
\end{definition}
\begin{proposition}[The pencil and the dynamics it generates]\label{prop:pencil}
(i) For every $(\alpha,\lambda)$, $L_u$ of equation~\ref{eq:pencil} is skew-adjoint, in the continuum and on the grid.
(ii) In the continuum, $\alpha D+\lambda(uD+Du)$ is Poisson for every $(\alpha,\lambda)$, because $D$ and $uD+Du$ are a compatible
Hamiltonian pair of the KdV hierarchy \citep{magri1978simple,olver1993applications}; on the grid, the Jacobi identity holds
exactly on alias-free states and test functionals. The wave operator is Poisson at every state.
(iii) The entropy is a Casimir of $L$: $L\,\delta S/\delta z=0$ at every state. For scalar fields, $\int u$ is a Casimir of the constant
part $\alpha D$ and of the wave operator, hence of $L$ when $\lambda=0$; it is not a Casimir of $uD+Du$, whose Casimir is
$\int\sqrt u$. Nevertheless the reversible flow generated by $E_\theta$ conserves $\int u$, because
$\delta E_\theta/\delta u=u+\varphi_\theta'(u)$ is a pointwise function of $u$ and $(uD+Du)g(u)$ is then a total derivative; on the
grid this is exact on alias-free states.
(iv) With $Q=\tfrac12\|u\|^2$ the reversible vector field is
\begin{equation}
L_u\,\frac{\delta E_\theta}{\delta u}=\alpha\,Du+3\lambda\,u\,Du+L_u\varphi_\theta'(u),
\label{eq:reversible-field}
\end{equation}
so linear advection $\partial_tu=-cDu$ is the member $(\alpha,\lambda)=(-c,0)$ and inviscid Burgers $\partial_tu=-uDu$ the member
$(0,-\tfrac13)$, the frame spacing being absorbed into the symbols.
\end{proposition}
\begin{proof}
(i)--(ii) are Proposition~\ref{prop:pencil-app}. (iii): $(0,1)$ lies in the $s$ block, which $L$ annihilates; $\alpha D1=0$ and the wave
operator vanishes at $k=0$, whereas $(uD+Du)1=Du\ne0$; for $g=g(u)$ and $G'(t)=tg'(t)$, $(uD+Du)g(u)=D\big(G(u)+ug(u)\big)$ has
zero mean, on the grid on alias-free modes. (iv): $(uD+Du)u=uDu+D(u^2)=3uDu$.
\end{proof}

We choose $\alpha D$ for the constant part of $L$ to ensure
compatibility with $uD+Du$. Using a freely learned Fourier multiplier
$a(k)$ as in GENERIC-FNO yielded a Jacobi residual of $O(10^{-1})$ even
on alias-free states (Remark~\ref{rem:compat}). The projectors in
Eq.~\ref{eq:pencil} de-alias quadratic products. The exactness guarantee
in (ii) requires that the grid resolve every cubic product in the Jacobi
identity. Setting $\lambda\equiv0$ gives a Poisson operator at every
state. We retain this ablation to isolate the effect of the
state-dependent transport term.
\begin{definition}[Friction operator]\label{def:friction}
Let $c=p+1$ be the number of state channels and introduce $r=c(d+1)$ \emph{slots}: for each channel $i$, one \emph{flux slot}
$(i,a)$ per axis $a$ and one \emph{local slot}. $D_M:L^2(\Omega;\mathbb R^r)\to L^2(\Omega;\mathbb R^c)$ is the Fourier multiplier with
symbols
\begin{equation}
\widehat{D_M}_{\,i,(j,a)}(k)=\frac{-ik_a\,b_{ija}(k)}{\sqrt{1+k_a^2b_{ija}(k)^2}},\qquad
\widehat{D_M}_{\,i,j}(k)=\frac{b'_{ij}(k)}{\sqrt{1+b'_{ij}(k)^2}},
\label{eq:DM}
\end{equation}
where $b,b'$ are learned, even, and vanish outside the band $|k|_\infty\le m_M$; the saturation bounds every multiplier by one.
A flux slot thus enters a channel through a divergence, a local slot through a bounded pointwise multiplier. Masks fix what is
conserved: for the scalar conservation laws the $u$ row of $D_M$ has flux columns only, so $(Mv)_u$ is a divergence for every
$v$ and $M$ conserves $\int u$; for the wave no slot feeds $q$, so friction acts on momentum and entropy only. The \emph{coefficient field}
$\Gamma_\vartheta(z):\Omega\to\mathbb R^{r\times r}$ is a pointwise matrix whose flux--flux and local--local blocks are the identity plus
bounded functions of $z(x)$ and whose mixed blocks are linear in the state gradients,
\begin{equation}
\Gamma_{(i,a),j}(x)=\sum_{l,b}h_{(i,a),j,(l,b)}\big(z(x)\big)\,\partial_bz_l(x),
\label{eq:gamma}
\end{equation}
and likewise for the local--flux blocks; all coefficient functions come from one bounded, zero-initialized pointwise network,
so $\Gamma_\vartheta=I$ at initialization and $\|\Gamma_\vartheta-I\|\le1$ throughout. Let $w(x)=(D_M^{\mathsf H}\,\delta E_\theta/\delta z)(x)\in\mathbb R^r$
be the energy gradient in slot space and $\mathsf P(x)=I-\hat w\hat w^{\!\top}$, $\hat w=w/|w|$. The friction operator is
\begin{equation}
M(z)=AA^{\!\top},\qquad A=D_M\,\mathsf P\,\Gamma_\vartheta(z).
\label{eq:friction}
\end{equation}
\end{definition}
\paragraph{Why this form.}
The definition follows the friction operators of physics. For heat conduction in the variables of
Definition~\ref{def:energy}, $M=CC^{\!\top}$ with $C^{\!\top}(f,g)=\sqrt\nu\,(\sqrt T\,\nabla f-g\,\nabla u/\sqrt T)$: a
derivative on the input side of $u$, a divergence on the output side, and a pointwise matrix in between. Applied to $(0,1)$ it
gives $\nu\Delta u$ on $u$ and $\nu|\nabla u|^2/T$ on $s$, and $M\,\delta E/\delta z=0$ holds because the rows of $C^{\!\top}$
are orthogonal to $(\nabla u,T)$ at every point, so degeneracy is local.
Equations~\ref{eq:DM}--\ref{eq:friction} contain this operator and its damped-wave analogue exactly
(Proposition~\ref{prop:represent}). Equation~\ref{eq:gamma} is linear in $\nabla z$ for a reason: the derivative orders in $M$
are then even, and $M$ cannot carry transport. Without this restriction the learned advection symbol converged to half its
value, and $M$ supplied the rest (Appendix~\ref{app:friction}).
\begin{proposition}[Structure for arbitrary parameters]\label{prop:structure}
For any parameters $(\theta,\eta,\vartheta,\alpha,\lambda,d,b,b')$ and any state, $L$ is skew-adjoint and Poisson (at every state if
$\lambda=0$, on alias-free states otherwise); $M(z)$ is self-adjoint and positive semi-definite; and
$L\,\delta S/\delta z=0$ and $M(z)\,\delta E_\theta/\delta z=0$, the latter pointwise. Consequently, solutions of
$\partial_tz=L\,\delta E_\theta/\delta z+M(z)\,\delta S/\delta z$ satisfy
\begin{equation}
\frac{\dd E_\theta}{\dd t}=0,\qquad
\frac{\dd S}{\dd t}=\Big\|A^{\!\top}\frac{\delta S}{\delta z}\Big\|^2\ge0.
\label{eq:laws}
\end{equation}
\end{proposition}
\begin{proof}
Propositions~\ref{prop:pencil-app}--\ref{prop:M} and Theorem~\ref{thm:metriplectic}; $A^{\!\top}\delta E_\theta/\delta z=\Gamma_\vartheta^{\!\top}\mathsf Pw=0$
by the definition of $\mathsf P$.
\end{proof}
\emph{Learned:} $\psi_\eta$, $\varphi_\theta$, $e_\theta$, the coefficient network of $\Gamma_\vartheta$, and the symbols $\alpha,\lambda$ (or $d$),
$b$, $b'$. \emph{Fixed:} the entropy, $Q$, the forms of $L$ and $M$, the masks and the integrator. The identities of
Proposition~\ref{prop:structure} hold before training and for every parameter value.
\subsection{One frame: integrating the dynamics}\label{sec:integration}

In components, $\partial_tu=L_u(u+\varphi_\theta'(u))+[M(0,1)]_u$
and $\partial_ts=[M(0,1)]_s$. The reversible flow acts only on $u$,
and $s$ evolves dissipatively. Each frame (Algorithm~\ref{alg:frame})
uses two exact unitary half-steps of the linear reversible flow around
$n_{\rm sub}=8$ explicit sub-steps for the Lie--Poisson, potential, and
friction terms. On each Fourier mode, $\mathcal U_\tau$ multiplies
scalar fields by $e^{i\alpha k_D\tau}$, where $k_D=\sum_a k_a$, and
rotates wave modes at frequency $d(k)c|k|$. The unitary map preserves the quadratic energy associated with $Q$ to
machine precision. The first-order remainder update yields an
$O(\Delta t^2)$ error in $E_\theta$ per frame. Outside training, we use
an implicit discrete-gradient step as the energy-error reference. It
conserves $E_\theta$ and satisfies the discrete entropy-production law
exactly (Theorem~\ref{thm:discrete}, Appendix~\ref{app:integration}).

\begin{algorithm}[t]
\caption{One frame of the Poisson-GENERIC operator, $z_{t+1}=\Phi_\theta(z_t)$}
\label{alg:frame}
\begin{algorithmic}[1]
\Require state $z=(u,s)$, frame spacing $\Delta t$, number of sub-steps $n_{\rm sub}$
\State $u\gets\mathcal U_{\Delta t/2}(u)$ \Comment{exact linear reversible half-step}
\For{$j=1,\dots,n_{\rm sub}$}
\State $g_E\gets\delta E_\theta/\delta z$ at $z$;\quad $g_S\gets(0,1)$
\State $r\gets Lg_E-L^{\rm const}\,\delta Q/\delta u$ \Comment{Lie--Poisson term and potential; the linear part is in $\mathcal U$}
\State form $\Gamma_\vartheta(z)$ and $\mathsf P$ from $w=D_M^{\mathsf H}g_E$
\State $r\gets r+D_M\mathsf P\Gamma_\vartheta\Gamma_\vartheta^{\!\top}\mathsf PD_M^{\mathsf H}g_S$ \Comment{friction $M(z)\,\delta S/\delta z$}
\State $z\gets z+(\Delta t/n_{\rm sub})\,r$
\EndFor
\State $u\gets\mathcal U_{\Delta t/2}(u)$
\State \Return $z$
\end{algorithmic}
\end{algorithm}
\subsection{Training objective}\label{sec:objective}
Every model, the baselines included, is trained to minimize the relative $L^2$ error of the observed field over a short
rollout from a random start,
\begin{equation}
\mathcal L(\theta)=\mathbb E_{\tau,t_0}\,\frac1H\sum_{h=1}^H
\frac{\big\|\big(\Phi_\theta^{\,h}(z^\tau_{t_0})\big)_u-u^\tau_{t_0+h}\big\|}{\big\|u^\tau_{t_0+h}\big\|},
\qquad z^\tau_{t_0}=\big(u^\tau_{t_0},\psi_\eta(u^\tau_{t_0})\big),
\label{eq:loss}
\end{equation}
with $H=4$ frames in 1D and $3$ in 2D, where $\Phi_\theta^{\,h}$ denotes $h$ applications of Algorithm~\ref{alg:frame} and
$(\cdot)_u$ the observed channels. $E_\theta$, $S$, $L$ and $M$ do not appear in the loss; the structure is in the
parameterization alone. The gradient is taken through the whole rollout, integrator included. The optimizer, schedule, batch
size and number of optimizer steps are the same for all models (Appendix~\ref{app:protocol}). The one exception is the operator
symbols, which are trained with a larger learning rate.
\subsection{Diagnostics}\label{sec:diagnostics}
The guarantees concern $E_\theta$ and $S$ and do not say whether the learned dynamics reproduce the governing equation. To
test this we use quantities that refer only to the fixed $Q$ and to the operators. With $u$ one frame and $u'$ the next,
these are:
\begin{enumerate}[label=(\alph*),nosep,leftmargin=*]
  \item the \emph{dissipation rate} $\Pi=(Q(u)-Q(u'))/Q(u)$ along the data ($\Pistar$) and along the model ($\Pimodel$),
        calibrated if $\Pimodel/\Pistar=1$;
  \item the \emph{reversibility cosine}
        $\rmech=-\langle\nabla Q(\bar u),u'-u\rangle/(\|\nabla Q(\bar u)\|\|u'-u\|)$ at the midpoint $\bar u$, zero when $Q$ is
        conserved over the frame;
  \item the normalized \emph{Jacobi residual} $\mathcal J$ of
        $\{F,G\}=\langle\delta F/\delta z,L\,\delta G/\delta z\rangle$ on random smooth test functionals;
  \item \emph{representability}, inserting the closed-form heat and wave friction operators into
        equation~\ref{eq:friction};
  \item the learned symbols against their exact values $\alpha=-c\Delta t$, $d\equiv\Delta t$ and $\lambda=-\Delta t/3$
        from Proposition~\ref{prop:pencil}(iv) (Appendix~\ref{app:diagnostics}).
\end{enumerate}
Diagnostic (e) revealed two errors that rollout accuracy had hidden, the gradient scaling of \S\ref{sec:problem} and the
transport split of \S\ref{sec:model}, and we therefore report the learned symbols with the results.
\section{Experiments}
\label{sec:experiments}

\paragraph{Setup.}
We use the four periodic equations of \citet{sulskis2026generic} in 1D and 2D: heat ($\nu=0.05$), advection ($c=1$), viscous
Burgers ($\nu=0.02$) and the damped wave ($c=1,\gamma=0.3$), the last observed as $(q,v)$. Initial conditions are random fields
on the \emph{data band}, $|k|_\infty\le8$ in 1D and $|k|_\infty\le6$ in 2D. Reference solutions come from a de-aliased spectral
solver, with frames $\Delta t=0.1$ apart on grids of $64$ and $32^2$ points. There are $120$ trajectories of $20$ frames in 1D
and $80$ of $15$ frames in 2D, split $80/20$ (Appendix~\ref{app:protocol}). Persistence, $u_{t+1}=u_t$, has an error of
$0.18$--$1.03$. FNO \citep{li2021fno}, Transolver \citep{wu2024transolver} and a periodic CNO-style operator are used both as the baseline
$u+G(u)$ and as the backbone of the structured models. We compare five models: the baseline; GENERIC-FNO
\citep{sulskis2026generic}, ported from its code and trained under our protocol and under its original recipe; the
learned-$E,S$ ablation, with a Casimir entropy, a constant $L$ and a learned energy (Appendix~\ref{app:ablations}); PG-phys with
$\lambda\equiv0$, which has the physical energy and exact Jacobi but no nonlinear transport; and PG-phys. The
\emph{compute-matched baseline} is an FNO trained for the wall-clock time PG-phys used on the same task. Errors are for
ten-frame rollouts on held-out trajectories, as mean $\pm$ standard deviation over three seeds. A \emph{seed-level comparison}
pairs two models trained on the same data and seed.

\paragraph{The identities hold before training.}
\label{sec:verification}
The identities of Proposition~\ref{prop:structure} are claimed for arbitrary parameters, so we check them at random
initialization, in double precision (Table~\ref{tab:structure}; full results in Table~\ref{tab:structure-full}). Skewness,
symmetry, the two degeneracy conditions, $\dot E_\theta=0$ and the conservation of $\int u$ by $M$ hold to rounding error for
every equation. After training, the mass of PG-phys drifts by at most $10^{-4}$ along a rollout, against
$10^{-3}$--$2\times10^{-2}$ for the baselines and $6\times10^{-2}$ for the learned-$E,S$ ablation. However, the Jacobi identity needs more care. The residual of the constant $L$ is $10^{-16}$ on every state. That of the pencil is as
small on alias-free states but $0.17$--$0.78$ on white noise, the aliasing of Proposition~\ref{prop:pencil}(ii), and
that of GENERIC-FNO's projected operator is $10^{-3}$--$1$ (Table~\ref{tab:structure-full}). On the states the experiments
visit (Table~\ref{tab:jacobi-data}), the pencil is Poisson on the resolved band only. Against data-band functionals the 1D
residual starts at the single-precision floor and reaches $7\times10^{-3}$ once Burgers has steepened; against full-grid
functionals it is $10^{-2}$--$0.9$ in 1D and $0.05$--$0.3$ in 2D, where the band is not cubically resolved.
Finally, the discrete-gradient step conserves $E_\theta$ to solver tolerance and the split step to $O(\Delta t^2)$, and the
closed-form heat and wave friction operators are reproduced to $3$--$7\times10^{-16}$ with
$M\,\delta E/\delta z\le10^{-17}$, without a global projection.

\begin{table}[t]
\centering\footnotesize
\setlength{\tabcolsep}{3pt}
\caption{Structural identities at random initialization in double precision
on a \(16^d\) grid, reported as the maximum over a batch. Here
``floor'' means \(\le10^{-16}\), and the mass of \(f\) is
\(|\sum_x f_u|/\|f\|\). Full results: Table~\ref{tab:structure-full}.}
\label{tab:structure}
\begin{tabularx}{\linewidth}{>{\raggedright\arraybackslash}X >{\raggedright\arraybackslash}p{2.5cm} c c c}
\toprule
identity & guarantees & target & 1D heat & 2D Burgers \\
\midrule
$\ip{v}{Lv}$ (skewness of $L$) & first law & floor & $\le10^{-16}$ & $\le10^{-16}$ \\
$\ip{w}{Mv}-\ip{v}{Mw}$ (symmetry of $M$) & second law & floor & $\le10^{-16}$ & $\le10^{-16}$ \\
$\|L\,\delta S/\delta z\|/\|\delta S/\delta z\|$ (Casimir entropy) & first degeneracy & $0$ & $0$ & $0$ \\
$\|M\,\delta E/\delta z\|/\|\delta E/\delta z\|$ (pointwise projector) & second degeneracy & floor & $3\cdot10^{-18}$ & $2\cdot10^{-18}$ \\
mass of $M\,\delta S/\delta z$ (mask) & $\int u$ conserved by $M$ & floor & $\le10^{-16}$ & $2\cdot10^{-16}$ \\
mass of $L\,\delta E_\theta/\delta z$, pencil, data-band state & $\int u$ conserved by the reversible flow & floor & $\le10^{-16}$ & $8\cdot10^{-16}$ \\
\midrule
$\ip{\delta E/\delta z}{\partial_tz}$ & $\dot E_\theta=0$ & floor & $\le10^{-16}$ & $\le10^{-16}$ \\
$\min\ \ip{\delta S/\delta z}{\partial_tz}$ & $\dot S\ge0$ & $\ge0$ & $0.014$ & $0.060$ \\
\midrule
$\mathcal J$, constant part of $L$, white-noise state & Poisson at every state & floor & $1\cdot10^{-16}$ & $3\cdot10^{-16}$ \\
$\mathcal J$, pencil ($\lambda=0.3$), alias-free / white-noise & Poisson on the band & floor / -- & $3\cdot10^{-16}$ / $0.78$ & $4\cdot10^{-16}$ / $0.17$ \\
$\mathcal J$, projected $L$ of GENERIC-FNO & (not Poisson) & -- & $0.38$ & $4.8\cdot10^{-2}$ \\
\midrule
$|E_1-E_0|/|E_0|$, discrete-gradient / split step & exact / $O(\Delta t^2)$ & tol.\ / small & $1\cdot10^{-11}$ / $5\cdot10^{-2}$ & $8\cdot10^{-11}$ / $5\cdot10^{-2}$ \\
\bottomrule
\end{tabularx}
\end{table}

\paragraph{Accuracy.}
On advection and the wave (Table~\ref{tab:accuracy}) PG-phys applies the closed-form
reversible flow, whereas the baseline has to learn the phase shift, and the error of PG-phys is lower by $3.6$--$58\times$ on
advection and $1.3$--$9.1\times$ on the wave. The learned symbols agree with the exact ones: $-ck\Delta t$ to three decimals in
all nine 1D advection runs, and $d(k)=0.100$--$0.102$ against $\Delta t=0.1$ in the eight 1D wave runs that trained.
The pencil was introduced for Burgers. With it the error is $2.0$--$5.5\times$ below the baseline and
$\lambda=-0.0332\pm0.0003$, against the exact $-\Delta t/3=-0.0333$. Without it ($\lambda\equiv0$) the model is worse than the
baseline in every Burgers cell, and three of its eighteen Burgers runs diverged. Heat is the hard case. PG-phys must build diffusion from a product of saturated
symbols and a coefficient field, where the baseline fits $e^{-\nu k^2\Delta t}$ with one multiplier. It wins three heat cells
and loses three. The exact operator is in the model class (Proposition~\ref{prop:represent}), so we attribute the losses to the
optimization; the $\lambda\equiv0$ model, with less to learn, is indeed better in all six heat cells. In total PG-phys is ahead of its baseline in $21$ of $24$ cells and $63$ of $72$ seed-level comparisons. The Transolver rows
compare equal budgets only, since that baseline has not converged (Appendix~\ref{app:protocol}).

Table~\ref{tab:ladder} shows where the gains come from. A Casimir entropy with a learned energy improves on GENERIC-FNO for
heat and advection, but not for Burgers or the wave. The physical energy and the exact linear step then improve advection and
the wave, and the pencil improves Burgers. PG-phys is better than GENERIC-FNO by $5$--$500\times$, and by $5$--$900\times$ in
GENERIC-FNO's metric on its own testbed (Table~\ref{tab:tmlr}).
All of this holds at equal optimizer steps and not at equal wall-clock time. The compute-matched FNO takes $36$--$47\times$
as many steps and is more accurate at ten frames on all four 1D equations, so the structure is paid for in training time
(Appendix~\ref{app:limitations}).

\begin{table}[t]
\centering\scriptsize
\setlength{\tabcolsep}{3pt}
\caption{Ten-frame relative \(L^2\) rollout error, mean \(\pm\) standard
deviation over three seeds. Parameter counts are baseline / PG-phys in
thousands. Bold marks the best of the four model classes; the wave has no
pencil, so its PG columns repeat the same model. \(\dagger\): one seed
diverged, and the entry reports the other two
(Appendix~\ref{sec:app_failures}). Transolver uses a separate 2D dataset.}
\label{tab:accuracy}
\begin{tabular}{llcccccr}
\toprule
& & Persist. & Baseline & Learned $E,S$ & PG-phys & PG-phys ($\lambda\equiv0$) & Params \\
\midrule
\multirow{3}{*}{1D heat}
& FNO & 0.184 & $\mathbf{0.0012\pm0.0002}$ & $0.0096\pm0.0067$ & $0.0161\pm0.0053$ & $0.0055\pm0.0007$ & 44 / 18 \\
& Transolver & 0.184 & $0.100\pm0.043$ & $0.035\pm0.014$ & $0.010\pm0.004$ & $\mathbf{0.0048\pm0.0010}$ & 34 / 42 \\
& CNO & 0.184 & $0.0137\pm0.0091$ & $0.0096\pm0.0031$ & $0.0063\pm0.0006$ & $\mathbf{0.0041\pm0.0013}$ & 21 / 12 \\
\midrule
\multirow{3}{*}{1D advection}
& FNO & 0.863 & $0.0040\pm0.0006$ & $0.0191\pm0.0066$ & $0.0011\pm0.0005$ & $\mathbf{0.0001\pm0.0001}$ & 44 / 18 \\
& Transolver & 0.863 & $0.037\pm0.003$ & $0.134\pm0.116$ & $0.0012\pm0.0008$ & $\mathbf{0.0001\pm0.0001}$ & 34 / 42 \\
& CNO & 0.863 & $0.0116\pm0.0015$ & $0.035\pm0.008$ & $0.0012\pm0.0004$ & $\mathbf{0.0001\pm0.0001}$ & 21 / 12 \\
\midrule
\multirow{3}{*}{1D Burgers}
& FNO & 0.645 & $0.090\pm0.010$ & $0.419\pm0.068$ & $\mathbf{0.033\pm0.008}$ & $0.338\pm0.018$ & 44 / 18 \\
& Transolver & 0.645 & $0.165\pm0.014$ & $0.355\pm0.028$ & $\mathbf{0.030\pm0.003}$ & $0.406\pm0.111^{\dagger}$ & 34 / 42 \\
& CNO & 0.645 & $0.072\pm0.005$ & $0.377\pm0.002$ & $\mathbf{0.035\pm0.003}$ & $0.302\pm0.099^{\dagger}$ & 21 / 12 \\
\midrule
\multirow{3}{*}{1D wave}
& FNO & 0.938 & $0.036\pm0.012$ & $0.813\pm0.645$ & $\mathbf{0.028\pm0.009}$ & $\mathbf{0.028\pm0.009}$ & 44 / 22 \\
& Transolver & 0.938 & $0.218\pm0.072$ & $0.805\pm0.418$ & $\mathbf{0.024\pm0.013}^{\dagger}$ & $\mathbf{0.024\pm0.013}^{\dagger}$ & 34 / 45 \\
& CNO & 0.938 & $0.104\pm0.024$ & $0.526\pm0.220$ & $\mathbf{0.036\pm0.003}$ & $\mathbf{0.036\pm0.002}$ & 21 / 16 \\
\midrule
\multirow{3}{*}{2D heat}
& FNO & 0.275 & $0.038\pm0.002$ & $0.146\pm0.063$ & $0.039\pm0.011$ & $\mathbf{0.025\pm0.004}$ & 445 / 128 \\
& Transolver & 0.280 & $0.040\pm0.011$ & $0.024\pm0.004$ & $0.028\pm0.016$ & $\mathbf{0.013\pm0.003}$ & 35 / 52 \\
& CNO & 0.275 & $0.036\pm0.010$ & $0.144\pm0.030$ & $0.039\pm0.016$ & $\mathbf{0.025\pm0.007}$ & 62 / 32 \\
\midrule
\multirow{3}{*}{2D advection}
& FNO & 0.960 & $0.074\pm0.007$ & $0.506\pm0.102$ & $0.0022\pm0.0012$ & $\mathbf{0.0001\pm0.0001}$ & 445 / 128 \\
& Transolver & 0.969 & $0.058\pm0.007$ & $0.199\pm0.182$ & $0.0010\pm0.0004$ & $\mathbf{0.0001\pm0.0000}$ & 35 / 52 \\
& CNO & 0.960 & $0.055\pm0.003$ & $0.504\pm0.024$ & $0.0022\pm0.0014$ & $\mathbf{0.0001\pm0.0001}$ & 62 / 32 \\
\midrule
\multirow{3}{*}{2D Burgers}
& FNO & 0.761 & $0.275\pm0.007$ & $0.568\pm0.017$ & $\mathbf{0.089\pm0.024}$ & $0.308\pm0.003^{\dagger}$ & 445 / 128 \\
& Transolver & 0.767 & $0.262\pm0.014$ & $0.561\pm0.034$ & $\mathbf{0.088\pm0.004}$ & $0.360\pm0.046$ & 35 / 52 \\
& CNO & 0.761 & $0.222\pm0.005$ & $0.586\pm0.055$ & $\mathbf{0.080\pm0.016}$ & $0.332\pm0.008$ & 62 / 32 \\
\midrule
\multirow{3}{*}{2D wave}
& FNO & 1.026 & $0.305\pm0.025$ & $0.922\pm0.156$ & $\mathbf{0.035\pm0.002}$ & $\mathbf{0.035\pm0.002}$ & 445 / 146 \\
& Transolver & 1.023 & $0.236\pm0.036$ & $0.810\pm0.482$ & $\mathbf{0.034\pm0.001}$ & $\mathbf{0.034\pm0.001}$ & 35 / 70 \\
& CNO & 1.026 & $0.301\pm0.024$ & $0.818\pm0.161$ & $\mathbf{0.035\pm0.002}$ & $\mathbf{0.035\pm0.002}$ & 62 / 50 \\
\bottomrule
\end{tabular}
\end{table}

\begin{table}[t]
\centering\footnotesize
\setlength{\tabcolsep}{3pt}
\caption{FNO model ladder: ten-frame rollout error, with three-seed standard
deviation in units of the last digit. Each column adds one design choice.
``Matched'' trains the baseline for the PG-phys wall-clock
(\(36\)--\(47\times\) more optimizer steps in 1D). GENERIC-FNO
\citep{sulskis2026generic} was ported from its code; shared weights reproduce
its forward pass to \(10^{-8}\). \(\dagger\): one seed diverged. Training
details: Appendix~\ref{app:protocol}.}
\label{tab:ladder}
\begin{tabular}{lccccccc}
\toprule
& Baseline & \pair{Baseline}{(matched)} & \pair{GENERIC-}{FNO} & \pair{GENERIC-FNO}{(own recipe)} & \pair{Learned}{$E,S$} & \pair{PG-phys}{($\lambda\equiv0$)} & PG-phys \\
\midrule
1D heat & $0.0012(2)$ & $0.0002(0)$ & $0.089(8)$ & $0.111(47)$ & $0.0096(67)$ & $0.0055(7)$ & $0.0161(53)$ \\
1D advection & $0.0040(6)$ & $0.0008(1)$ & $0.534(120)$ & $0.666(54)$ & $0.0191(66)$ & $0.0001(1)$ & $0.0011(5)$ \\
1D Burgers & $0.090(10)$ & $0.030(2)$ & $0.556(49)$ & $0.574(41)$ & $0.419(68)$ & $0.338(18)$ & $0.033(8)$ \\
1D wave & $0.036(12)$ & $0.0026(3)$ & $0.891(38)$ & $0.948(43)$ & $0.813(645)$ & $0.028(9)$ & $0.028(9)$ \\
2D heat & $0.038(2)$ & -- & $0.183(7)$ & $0.170(5)$ & $0.146(63)$ & $0.025(4)$ & $0.039(11)$ \\
2D advection & $0.074(7)$ & -- & $0.923(51)$ & $0.844(26)$ & $0.506(102)$ & $0.0001(1)$ & $0.0022(12)$ \\
2D Burgers & $0.275(7)$ & -- & $0.690(38)$ & $0.701(17)$ & $0.568(17)$ & $0.308(3)^{\dagger}$ & $0.089(24)$ \\
2D wave & $0.305(25)$ & -- & $1.02(1)$ & $1.03(1)$ & $0.922(156)$ & $0.035(2)$ & $0.035(2)$ \\
\bottomrule
\end{tabular}
\end{table}

\begin{figure}[t]
\centering
\includegraphics[width=\linewidth]{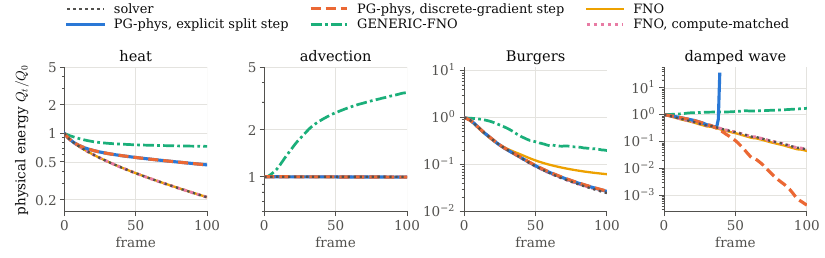}
\caption{Physical energy \(Q_t/Q_0\) over 100-frame rollouts
(1D, FNO, seed 0, \(24\) held-out trajectories). The explicit PG-phys wave
step diverges at frame \(39\); the discrete-gradient step remains bounded but
over-damps the wave (Appendix~\ref{sec:integrator_study}).}
\label{fig:energy}
\end{figure}

\paragraph{Energy and entropy over one hundred frames.}
By Proposition~\ref{prop:structure}, $Q$ can only be transferred to $\int e_\theta(s)$, irreversibly and at a learned rate,
and it is conserved when there is no dissipation. We test this with four 1D FNO models (seed 0) on inputs from the training
distribution, with twice the training amplitude, and with the wider band $|k|\le12$ (Figure~\ref{fig:energy};
Table~\ref{tab:relevance} in Appendix~\ref{app:extended}). PG-phys never raises $Q$ by more than $0.1\%$ on Burgers and advection for all three input sets, or on heat as long as the
explicit integrator is stable. On advection it keeps the energy within $0.5\%$ for 100 frames in and out of distribution,
whereas at twice the amplitude the unconstrained FNO loses $21\%$ of it and the compute-matched FNO $26\%$; extra training
helps the FNO only in distribution. On Burgers the FNO at equal optimizer steps dissipates $2.5\times$ too slowly. PG-phys also reproduces
the dissipation rate (Table~\ref{tab:calibration}): it recovers $0.90$--$1.03$ of the true heat and Burgers dissipation in
eleven of twelve cells and $0.76$--$0.94$ on the wave, with $|\Pimodel|\le7\times10^{-4}$ on advection. It is not enough for \emph{some} energy and entropy to satisfy the identities. GENERIC-FNO satisfies degeneracy for its
learned $E$ and $S$, yet its physical energy increases on $63$--$99$ of 100 frames on advection and $71$--$87$ on the wave,
up to $3.5\times$ and $34\times$ the solver's, and it fails by frames $6$--$20$ on three equations. In the learned-$E,S$
ablation the identities hold to $10^{-8}$, but the wave energy grows twelvefold in ten frames. This is why we combine the
Casimir entropy with the physical energy.

However, a model that obeys the two laws need not be accurate. In distribution the compute-matched FNO has the lower error at
100 frames on all four equations. Out of distribution PG-phys is better on advection and worse on the other three. With a
quarter of the training data it degrades less on advection, Burgers and the wave, by factors of $1.4$--$9$, and more on heat. PG-phys fails in two ways (Appendix~\ref{app:limitations}). First, the heat energy decreases monotonically but too slowly: at
frame 100 it is $2.2\times$ the solver's. The coefficient field is fitted on a short window in which the fast modes dominate,
and it under-damps the slow modes that remain later. Second, with the explicit split step the wave rollout diverges for two
of the three seeds, at frames $39$ and $29$, as do the heat and wave rollouts at twice the amplitude, and more sub-steps do
not cure this. Solving the discrete-gradient step by Newton--Krylov does: all in-distribution wave rollouts of the three
seeds reach frame 100 without any increase of $Q$, at $20\times$ the cost (Appendix~\ref{sec:integrator_study}).

\paragraph{Ablations, cost, and scope.}
Two ablations each remove one component of the model, the physical energy or the pencil. Without the physical energy (learned-$E,S$) the identities still hold but no longer
refer to the physics, and the model wins only $2$ of $48$ seed-level comparisons on the FNO and CNO cells. Without the pencil
($\lambda\equiv0$) the Jacobi identity holds exactly at every state and the linear equations are solved
better, to $10^{-4}$ on advection and in all six heat cells, but without nonlinear transport every Burgers cell is lost and
three Burgers runs diverge. and the linear equations are solved better, to $10^{-4}$ on advection and
in all six heat cells, but without nonlinear transport every Burgers cell is lost and three Burgers runs diverge. The pencil
thus gives up some linear accuracy, and the Jacobi identity on aliased modes, in exchange for Burgers. PG-phys transfers from $64$ to
$128$ and $256$ points without retraining and has $0.3$--$0.8\times$ the parameters of the FNO and CNO baselines
(Appendix~\ref{app:extended}). Training is $24$--$46\times$ slower, because the coefficient field is re-evaluated and
differentiated through in each of the eight sub-steps. The guarantees concern the model's own functionals and, for the
pencil, alias-free states; agreement with the target equation is an empirical question (Appendix~\ref{app:limitations}).
\section{Conclusion}
\label{sec:conclusion}

GENERIC-FNO enforced degeneracy by projection, which broke the Jacobi identity. Here we add an entropy density to the state
and let the Poisson operator act only on the observed field, so that the entropy is a Casimir, and a factored friction
operator enforces the second degeneracy condition pointwise. With the physical energy and learned constitutive functions,
the model conserves total energy and has nonnegative entropy production for every parameter choice; the Jacobi identity
holds on alias-free states, and the discrete-gradient integrator preserves both laws at each step. The aim was a
thermodynamically consistent neural operator that preserves the physics, and higher accuracy than standard operators was
not the goal. At equal optimizer steps the model outperformed unconstrained operators in $63$ of $72$ seed-level comparisons,
but at equal wall-clock time the unconstrained FNO was more accurate. Across stable $100$-frame rollouts, in and out of
distribution, the physical energy never increased and stayed constant when dissipation vanished, whereas the unconstrained
operator lost energy on a conservative test out of distribution, and the variant with learned energy and entropy gained
physical energy despite satisfying the same identities. Two limits remain: on the heat equation the model dissipated energy
too slowly beyond the training window, and on the wave equation the explicit integrator diverged, where the implicit one
remained stable at twenty times the cost.

\subsection*{Broader Impact}
The model conserves its own energy and produces its own entropy in continuous time; in discrete time this is exact only with the discrete-gradient step, which we solve by Newton--Krylov at $4$--$24\times$ the cost of the split step per frame. We wish to stress that these are properties of the parameterization and that they certify nothing about the physics. A model of this form applied to a forced, open or otherwise non-conservative system, which is outside the scope of GENERIC, would still report residuals at machine precision. An example can be found in this paper: the learned-energy ablation conserves its energy to $10^{-8}$ while the physical energy of the wave grows twelve-fold within ten frames. Exact structure should therefore never be taken as validation, particularly in safety-critical simulation. The relevant check is the calibration against the true dissipation, and it presupposes a target system that belongs to the class. What we hope the method provides is a surrogate whose energy balance remains physical far from the training data. It does not provide more accuracy per unit of compute, since training costs $12$--$46\times$ that of an unconstrained operator, depending on the backbone, and at equal wall-clock time the unconstrained operator is the more accurate one over ten frames. The limits reported above hold for any use of the model: the dissipation rate is miscalibrated beyond the training window, and the explicit integrator can diverge where the implicit one does not.

\subsection*{AI use statement}

In this work, we used generative AI tools to assist with verifying mathematical proofs and derivations, identifying and reviewing relevant literature, cross-checking citations, and comparing technical claims and research ideas against the existing literature. Generative AI was also used to refine and condense portions of the manuscript, particularly to improve clarity and meet the nine-page limit without changing the intended technical content.

We additionally used Claude Code to assist with implementing the Python experimental framework from author-specified mathematical formulations, architectural designs, algorithms, and experimental requirements. AI-generated and AI-assisted code was reviewed, tested, and validated against the intended formulations and expected numerical behavior. AI-assisted literature and citation suggestions were independently checked against the original sources, and AI-assisted mathematical reasoning was reviewed for correctness. The authors determined the research questions, methodology, model architecture, experimental design, initial skeleton code, proof sketches, the interpretation of the results, and the scope/limitations and conclusions.

We have reviewed all AI-assisted work and take responsibility for the final content of this work, including text, claims, code, and other artifacts produced with the aid of generative AI.

\subsection*{Reproducibility Statement}
Every number in Tables~\ref{tab:structure} and \ref{tab:accuracy} outside the Transolver rows is read from one fresh run of a single self-contained notebook (PyTorch only) that generates the data, trains all four models on every cell for three seeds, evaluates the structural identities in double precision, runs the representability, super-resolution and integrator checks, and writes a single review file from which the tables were built. The GENERIC-FNO port, the compute-matched baselines, the Jacobi residuals on data states (Tables~\ref{tab:ladder} and~\ref{tab:jacobi-data}) and the run on GENERIC-FNO's testbed (Table~\ref{tab:tmlr}) come from a second and a third notebook built on the same code. A fourth notebook produces the integrator study on seed 0 (Figure~\ref{fig:integrators}), and a fifth the Transolver rows of Table~\ref{tab:accuracy}, the long-horizon, data-efficiency and off-distribution tests of Table~\ref{tab:relevance} with Figure~\ref{fig:energy}, the three-seed integrator study of Table~\ref{tab:integrators} and its extension to the other backbones and to two dimensions; both save every model they train. Three short notebooks read the learned symbols from the saved models, train the Burgers models needed to report $\lambda$ over all eighteen runs, and repeat the Transolver baseline at five times the budget (Appendix~\ref{app:protocol}); the remaining learned symbols come from the review file and the training-time ratios from single-cell runs of the same code. All notebooks, with the outputs of the runs reported here, are in the supplementary material \url{https://github.com/jaysulk/Poisson-GENERIC-Neural-Operators}; the result files of these runs are at \url{https://drive.google.com/drive/folders/1eDdsiUmixqbVZXtM4EUGa9kGvtzFVaG5}. Appendix~\ref{app:protocol} gives the architectures, data generation and training protocol, Appendix~\ref{app:construction} the implementation of the construction, and Appendix~\ref{sec:app_proofs} the proofs. Two implementation details that a reimplementation is likely to get wrong are called out in the text: the grid-measure inner product with Riesz-scaled gradients (Appendix~\ref{app:innerproduct}) and the separate learning rate for the operator symbols (Appendix~\ref{app:protocol}).

\bibliography{iclr2027_conference}
\bibliographystyle{iclr2027_conference}

% =====================================================================
% Appendices. A: implementation of the construction (nothing defined in Sec. 3 is defined again).
% B: architecture, data and training protocol, including the GENERIC-FNO port and the compute-matched baseline.
% C: statements and proofs. D: extended results, and limitations and scope (moved from the main text).
% =====================================================================
\appendix

\section{Details of the construction}
\label{app:construction}
\label{sec:app_method}

The model is defined by Definitions~\ref{def:state}--\ref{def:friction}. In this appendix we collect the implementation choices
behind those definitions, together with the derivations that \S\ref{sec:method} refers to. No definition is repeated.

\subsection{Inner product and gradient scaling}
\label{app:innerproduct}

Every functional in the model is implemented as a mean over the grid, $F[z]=\frac1N\sum_xf(z(x),\dots)$. With respect to the
grid-measure inner product of \S\ref{sec:problem}, its gradient is $N$ times the gradient returned by automatic differentiation,
since $\ip{\delta F/\delta z}{h}=\frac1N\sum_x(\delta F/\delta z)\cdot h$ has to equal the directional derivative $\sum_x(\partial F/\partial z(x))\cdot h(x)$. We
use this scaling for $E_\theta$ and $S$, and the same inner product in every projection and symbol. The scaling is easy to get
wrong, and we mention two ways of doing so because their symptoms are quite different. If the mean is taken per \emph{channel},
the gradients of $Q$ on the $p$ observed channels and of $E_\theta$ on the $p+1$ channels of $z$ differ by a factor $(p+1)/p$.
The exact linear step of Algorithm~\ref{alg:frame} then integrates only a part of the linear flow, and the learned reversible
symbols converge to an integer fraction of their exact values, one half for scalar fields and one third for the wave, whatever
the seed and the backbone. If instead the Euclidean gradient is used (no factor $N$), the model is no longer independent of the
resolution: the same 1D-heat model evaluated at $64/128/256$ points had errors $0.011/0.078/0.125$, compared with
$0.0183/0.0183/0.0183$ with the correct scaling (\S\ref{sec:integrator_study}).

\subsection{Internal energy and the gauge projection}
\label{app:energy}

\paragraph{Internal energy.} We take $e_\theta(s)=T_0s+\sum_{i=1}^{m}w_i\operatorname{softplus}(a_is+b_i)$ with $m=16$,
$w_i=\operatorname{softplus}(\tilde w_i)$, $a_i=\operatorname{softplus}(\tilde a_i)$ (initialized so that $w_i\approx0.05$) and $T_0=0.1$.
Since $w_i,a_i>0$ this function is convex and increasing, and $T(s)=T_0+\sum_iw_ia_i\,\sigma(a_is+b_i)\ge T_0$. The floor is not a
technicality. The friction coefficients of Proposition~\ref{prop:represent} contain $1/T$, and without a floor a negative drift
of $s$ early in training drives $T$ to zero, at which point the split step breaks down.

\paragraph{Gauge projection.} The function $\psi_\theta=\operatorname{softplus}(\operatorname{MLP}_\theta(u))$ is given by a $p\to32\to32\to1$ GELU network.
The projection $\Pi_{\le2}$ of equation~\ref{eq:potential} is a weighted least-squares projection onto the monomials of degree $\le2$ in the
components of $u$, that is, the constant, the $p$ linear and the $\binom{p+1}{2}$ quadratic monomials. The weights $e^{-|u|^2/2}$
correspond to a Gaussian reference measure of unit variance, and the data have unit RMS. In practice the projection is computed
on a fixed tensor-product grid of $17$ nodes per component on $[-4,4]$: at each forward pass we evaluate $\psi_\theta$ at the
nodes, solve the weighted normal equations (a small linear system) for the polynomial coefficients, and subtract the resulting
polynomial from $\psi_\theta(u)$ pointwise. One might think that it suffices to subtract the Taylor expansion of $\psi_\theta$
at $u=0$. It does not, because the quadratic that the network learns need not be centered at zero, and the remaining gauge shows
up as a reversible symbol that converges to $(1+2a)^{-1}$ times its exact value (Remark~\ref{rem:gauge}). Finally, writing
$\Pi_{\le2}\psi_\theta(u)=u^{\!\top}\!Au+b\cdot u+c_0$ and $a:=\lambda_{\max}(A)$, we have
$\varphi_\theta\ge-\Pi_{\le2}\psi_\theta$, which is a quadratic lower bound of curvature $-a$. The scalar mechanical energy is
therefore bounded for $a<\tfrac12$; the wave requires the additional hypotheses of Theorem~\ref{thm:bound}. Note that $a$ can
be measured from the learned symbols (Remark~\ref{rem:gauge}).

\subsection{The Poisson operator}
\label{app:poisson}

\paragraph{Symbols.} The constant part $\alpha D$ has the symbol $i\alpha k_D$, $k_D=\sum_ak_a$, on the band $1\le|k|_\infty\le m_L$ and vanishes at
$k=0$ and outside the band. Being odd in $k$, it defines a real and skew operator. For the wave block
$J\operatorname{diag}(D_d,D_d)$ we use an even, non-negative multiplier $d(k)$, parameterized as the square of a learned even
symbol. This block is Poisson because its coefficients are constant, and it is skew because $D_d$ is symmetric and commutes
with $J$. For the exact damped wave $d\equiv\Delta t$, and the learned multiplier agrees with this value to within $0.5\%$
(\S\ref{sec:experiments}). In both cases $\ker L$ contains every $s$-direction, at every state.

\paragraph{The pencil.} The operators $D$ and $uD+Du$ are the two Hamiltonian operators of the KdV hierarchy in the direction $D$
\citep{magri1978simple,olver1993applications}. They are compatible, which means that every linear combination of them is
Poisson; this is Proposition~\ref{prop:pencil}(ii). It is the Lie--Poisson term that makes Burgers Hamiltonian with respect to
the quadratic energy, since
$-\partial_x(u^2/2)=-\tfrac13(u\partial_x+\partial_xu)\,\delta(\tfrac12\|u\|^2)/\delta u$. On the grid we use the symmetric form $uDg+D(ug)$,
which is exactly skew there (Proposition~\ref{prop:pencil-app}(i)). The $2/3$-rule sandwich $P(\cdot)P$ does not destroy skewness,
because $P$ is symmetric, and it removes the aliasing of the quadratic products $uDg$ and $ug$. The cubic products that occur
in the Jacobi identity are not de-aliased, however. For this reason Proposition~\ref{prop:pencil}(ii) holds on alias-free states
only, the residual on white-noise states being $0.17$--$0.78$ (Table~\ref{tab:structure-full}), and we claim a Poisson structure
on the resolved band and nowhere else. The hierarchy has a further compatible member, the dispersive operator $D^3$, which we
have left out. Its $k^3$ growth destabilizes both integrators of \S\ref{app:integration} at high wavenumbers, and none of our
test equations needs it.

\begin{remark}[A free multiplier is not compatible]
\label{rem:compat}
Replacing $\alpha D$ by a free learned multiplier $a(k)$, as in GENERIC-FNO, gives an operator $a(k)+\lambda(uD+Du)$ whose Jacobi
residual is $O(10^{-1})$ even on alias-free states: only the members of the pencil spanned by $D$ and $D^3$ are compatible with
$uD+Du$. This is why the constant part of $L$ is the one-parameter family $\alpha D$.
\end{remark}

\subsection{The friction operator}
\label{app:friction}

\paragraph{The physical operators.} For heat conduction in the variables of Definition~\ref{def:energy}, $E=\tfrac12\|u\|^2+\int e(s)$,
$S=\int s$, $T=e'(s)$, the positive semi-definite operator
\begin{equation}
M=CC^{\!\top},\qquad
C^{\!\top}(f,g)=\sqrt{\nu}\Big(\sqrt{T}\,\nabla f-\frac{\nabla u}{\sqrt{T}}\,g\Big)
\label{eq:heatM}
\end{equation}
satisfies $[M\,\delta S/\delta z]_u=\nu\Delta u$, $[M\,\delta S/\delta z]_s=\nu|\nabla u|^2/T$ (viscous heating) and $M\,\delta E/\delta z=0$
identically. The analogue for the damped wave contains no derivatives. It is $C^{\!\top}(f_v,g)=\sqrt\gamma\,(\sqrt Tf_v-v\,g/\sqrt T)$, and it gives
$-\gamma v$ on $v$ and $\gamma v^2/T$ on $s$. Three features of equation~\ref{eq:heatM} are worth pointing out. The operator
depends on the state and on $\delta E/\delta z$. The derivative is on the input side of the $u$ channel and the divergence on the
output side, with a pointwise matrix between them. Finally, the row of $C^{\!\top}$ is orthogonal to $(\nabla u,T)$ at each
point, and this is how degeneracy can hold without a global projection. Both operators are contained in
Definition~\ref{def:friction} exactly on the band $\nu k^2\Delta t<1$, which includes the data band
(Proposition~\ref{prop:represent}). For heat one takes the flux symbol $-ik_a\kappa$ on the $u$ row, the identity
on the local slot of $s$, $\Gamma_{(u,a),(u,a)}=\sqrt{\nu T}/\kappa$ and $\Gamma_{s,(u,a)}=-\sqrt\nu\,\partial_au/(\kappa\sqrt T)$; for the wave,
$D_M$ is the identity on the local slots of $v$ and $s$, with $\Gamma_{v,v}=\sqrt{\gamma T}$ and $\Gamma_{s,v}=-\sqrt\gamma\,v/\sqrt T$.

\paragraph{Symbols, saturation and masks.} The symbols $b_{ija}(k)$ and $b'_{ij}(k)$ of equation~\ref{eq:DM} are learned on the band
$|k|_\infty\le m_M$ and then symmetrized, $b(k)\leftarrow\tfrac12(b(k)+b(-k))$, so that $D_M$ maps real fields to real fields. The saturation
$x\mapsto x/\sqrt{1+x^2}$ bounds every symbol by one in modulus. As a consequence the gain of $M$ under explicit integration is
limited for any parameters, while the physical operators are not affected on the data band, where $\nu k^2\Delta t\le0.32$. As
for the masks of Definition~\ref{def:friction}, for the scalar equations they set the local-slot inputs of the $u$ row to zero,
so that no local slot feeds $u$ and a density cannot be damped in the manner of Rayleigh friction; for the wave they set the
whole $q$ row to zero.

\paragraph{Coefficient field.} The field $\Gamma_\vartheta$ is produced by a pointwise network whose inputs are $z(x)$ and, for the
flux--local and local--flux blocks, the gradient features $\partial_bz_l(x)$ of equation~\ref{eq:gamma}. Its output passes through a
$\tanh$ scaled by $1/r$, which gives $\|\Gamma_\vartheta-I\|\le1$, and it is initialized at zero, so that $\Gamma_\vartheta=I$ when
training starts. The bound is needed. With unbounded coefficients the local dynamics of $s$ can become stiff; in one preliminary
run $s$ grew by $40\%$ per step and the explicit remainder blew up on advection. Nothing physical is lost, since physical
friction coefficients are bounded functions of the state.

\paragraph{Why the parity restriction.} Suppose that the mixed blocks of $\Gamma_\vartheta$ were allowed to depend on $z$ alone, and
not linearly on $\nabla z$. The operator $D_M\Gamma\Gamma^{\!\top}D_M^{\mathsf H}$ would then contain terms with an odd number of
derivatives, and such terms act as transport. This is what happened in a preliminary version without the restriction: the
learned advection symbol converged to exactly half of its exact value, for every seed and backbone, and $M$ supplied the other
half. With the restriction in place the transport has to be carried by the reversible operator, and the symbols converge to the
exact values (\S\ref{sec:experiments}). What the restriction excludes, and what it does not, is stated in
Proposition~\ref{prop:parity}.

\paragraph{Pointwise degeneracy.} Let $w=D_M^{\mathsf H}\delta E_\theta/\delta z$ and $\mathsf P=I-\hat w\hat w^{\!\top}$ at each point. Then
$A^{\!\top}\delta E_\theta/\delta z=\Gamma_\vartheta^{\!\top}\mathsf Pw=0$ pointwise, and hence $M\,\delta E_\theta/\delta z=0$ without any global projection
(Proposition~\ref{prop:M}). The heat-conduction row of equation~\ref{eq:heatM} is of this type, since $(\sqrt T,-\nabla u/\sqrt T)$ is
orthogonal to $(\nabla u,T)$ at each point. The global projection of GENERIC-FNO is thus not needed, and with $S=\int s$ it would
in fact be harmful. The reason is that when $T$ is nearly uniform, $\delta S/\delta z=(0,1)$ is nearly parallel to the entropy
component of $\delta E_\theta/\delta z=(\cdot,T)$, so that projecting off the latter removes most of the dissipative direction.

\subsection{Time integration}
\label{app:integration}
\label{sec:integrator}

\paragraph{Exact linear flow and split step.} The constant part of $L_u$ together with the quadratic $Q$ generates the flow
$\partial_tu=L_u^{\rm const}\,\delta Q/\delta u$. This flow is diagonal in Fourier space. For scalar fields it multiplies each mode by $e^{i\alpha k_D\tau}$,
and for the wave it is a $2\times2$ rotation of $(\hat q,\hat v)$ with frequency $\omega=d(k)\,c|k|$. The map $\mathcal U_\tau$ in
Algorithm~\ref{alg:frame} applies it exactly, at the cost of one FFT pair, and conserves $Q$ to machine precision.
Algorithm~\ref{alg:frame} as a whole is a Strang arrangement \citep{strang1968construction}, namely $\mathcal U_{\Delta t/2}$,
followed by $n_{\rm sub}=8$ explicit Euler sub-steps of the remainder, followed by $\mathcal U_{\Delta t/2}$ again. By the
remainder we mean the Lie--Poisson term, the potential and the friction; the linear part of $L\,\delta E_\theta/\delta z$ is
subtracted from it, since $\mathcal U$ has already applied it. The explicit remainder is kept stable by the saturated symbols
and by the sub-cycling. Since the split step is only first order in $\Delta t$ for fixed $n_{\rm sub}$, its energy error does not
vanish, and we report it (Table~\ref{tab:structure}). All training and evaluation use this integrator. We note in passing that
it may also explain why the unconstrained FNO won on advection in GENERIC-FNO: the exact phase shift is a map which the FNO has
to learn, whereas here it is simply applied.

\paragraph{Discrete-gradient step.} Let $\dg E(z,z')$ be the Gonzalez discrete gradient of $E_\theta$
\citep{gonzalez1996time,mclachlan1999geometric},
\begin{equation}
\dg E(z,z')=\frac{\delta E_\theta}{\delta z}(\bar z)
+\frac{E_\theta[z']-E_\theta[z]-\ip{\frac{\delta E_\theta}{\delta z}(\bar z)}{z'-z}}{\|z'-z\|^2}\,(z'-z),
\qquad\bar z=\tfrac12(z+z'),
\label{eq:dg}
\end{equation}
which satisfies $\ip{\dg E}{z'-z}=E_\theta[z']-E_\theta[z]$; since $S$ is linear, $\dg S=(0,1)$ exactly. The implicit update
\begin{equation}
z_{t+1}=z_t+\Delta t\Big(L(\bar z)\,\dg E(z_t,z_{t+1})+\bar M\,\dg S\Big),\qquad
\bar M=\bar A\bar A^{\!\top},\quad \bar A=D_M\,\bar{\mathsf P}\,\Gamma_\vartheta(z^*),
\label{eq:dgstep}
\end{equation}
with $\bar{\mathsf P}$ the pointwise projector off $D_M^{\mathsf H}\dg E$, conserves $E_\theta$ and produces $S$ exactly for any
$\Delta t$ (Theorem~\ref{thm:discrete}). The only property used in the proof is that $\dg S$ lies in $\ker L(z)$ for every $z$, and
this is true because $S$ is a Casimir of every $L$ of block form, the pencil evaluated at $\bar z$ included. The scheme can be
regarded as the discrete counterpart of the thermodynamically consistent integrators of \citet{romero2009thermodynamically} and
\citet{ottinger2018generic}. It is not available for the projected operator of GENERIC-FNO, whose kernel moves with the state
(Remark~\ref{rem:kernel}). There is some freedom in the point
$z^*$ at which the coefficient field is evaluated. Indeed, Theorem~\ref{thm:discrete} only requires $\bar M$ to be symmetric positive
semi-definite with $\bar M\dg E=0$, and both properties hold wherever $\Gamma_\vartheta$ is evaluated, since
$\bar A^{\!\top}\dg E=\Gamma_\vartheta^{\!\top}\bar{\mathsf P}D_M^{\mathsf H}\dg E=0$ by the construction of $\bar{\mathsf P}$.

\paragraph{Solving the implicit step.} In the submitted version equation~\ref{eq:dgstep} is solved by a relaxed fixed-point
iteration started from the explicit predictor ($3$--$12$ iterations), with $z^*=\bar z$. This iteration converges in all the
random-initialization checks of Table~\ref{tab:structure-full}, and also for single steps of the trained 1D and 2D heat FNO
models of \S\ref{sec:integrator_study}. It does not always converge, however. It fails on the trained 2D-heat Transolver models
at $\Delta t=0.1$, whether started from the Euler or from the split predictor; it fails on the trained 1D wave model, where the
residual grows with the number of iterations; and it fails as a hundred-frame integrator on three of the four 1D models. We
stress that the backbone is not part of this iteration. The backbone enters the model only through the encoder of the initial
entropy density (Definition~\ref{def:state}), and $\Gamma_\vartheta$ is a pointwise network. The failures must therefore be
attributed to the trained operators and states, and not to an architecture inside the solve. Models from two backbone families
differ both in the entropy density produced by their encoders and in their trained coefficients, and we have not been able to
determine which of the two is decisive in the Transolver case. The restriction disappears when the step is solved by
Newton--Krylov (\S\ref{sec:integrator_study}), that is, by an inexact Newton iteration on the residual of
equation~\ref{eq:dgstep}, with GMRES applied to Jacobian--vector products computed by reverse-over-reverse automatic
differentiation, backtracking on the residual norm, and halving of any step whose solve fails. Solved in this way, the
discrete-gradient step can be used as a hundred-frame integrator for the trained 1D models and as a sixty-frame integrator for
the 2D and Transolver models we tested, with $\Gamma_\vartheta$ evaluated at the midpoint. Freezing $\Gamma_\vartheta$ at a
predictor would also leave both laws exact, as shown above, but this turned out to be unnecessary. In the experiments of
Section~\ref{sec:experiments} the discrete-gradient step is the reference against which the energy error of the split step is
measured. One could also train with it, exactly, at several times the cost.

\subsection{Diagnostics}
\label{app:diagnostics}

All diagnostics refer to the fixed mechanical energy $Q$, never to a learned functional. (a) $\Pistar$ is the fractional one-frame
decay of $Q$ along the reference trajectory, and $\Pimodel$ is the same quantity along the model's rollout; both are averaged
over frames and trajectories. (b) With the midpoint gradient, the numerator of $\rmech$ equals $Q(u)-Q(u')$ exactly when $Q$ is
quadratic, so that a frame over which $Q$ is conserved scores zero. This is not the case for the endpoint gradient used in GENERIC-FNO, which
carries a bias of $O(\|\Delta u\|/\|u\|)$, about $0.1$ at our frame spacing, even for the exact advection solution. Being a
cosine, $\rmech$ under-weights dissipation on transport-dominated flows, and we therefore report magnitudes through
$\Pimodel/\Pistar$. (c) For the Jacobi residual we use three random smooth test functionals with non-trivial second variations,
$F[z]=\sum_j\alpha_j\tanh(\ip{z}{r_j}/\sqrt N)+\tfrac12\|z\|^2/N$ with random fields
$r_j$ (band-limited for the alias-free test), and evaluate them by automatic differentiation. The residual is
$\mathcal J=|\{\{F,G\},H\}+\mathrm{cyc}|/\sum_{\mathrm{cyc}}|\{\{F,G\},H\}|$, so that $\mathcal J=1$ means that there is no cancellation at all. It is
computed on white-noise states, on band-limited states with band-limited functionals and, in Table~\ref{tab:jacobi-data}, on
held-out data states with functionals on the data band and on the full grid. The projected operator of GENERIC-FNO is built
from an unrestricted learned entropy of the same size, because with a Casimir entropy the projection is trivial and there is
nothing to compare. (d) For representability we set $D_M$ and $\Gamma_\vartheta$ inside the model to the closed-form coefficients
of \S\ref{app:friction} and compare $M\,\delta S/\delta z$ and $M\,\delta E/\delta z$ with the spectral truth on band-limited states. (e) The learned
symbols $\alpha$, $d(k)$ and $\lambda$ are read from the trained models and compared with $-c\Delta t$, $\Delta t$ and $-\Delta t/3$. It
was this diagnostic that uncovered the gradient-scaling error of \S\ref{app:innerproduct} and the transport split of
\S\ref{app:friction}, and it also provides the empirical gauge diagnostic of Remark~\ref{rem:gauge}.

\subsection{Ablation variants}
\label{app:ablations}

The \emph{learned-$E,S$} model is the closest analogue of GENERIC-FNO that has a constant $L$. Both $E_\theta$ and $S_\phi$ are
mean-pooled neural operators of the backbone family. The entropy $S_\phi$ sees $u$ only through the modes outside the band of
$L$ and is therefore a Casimir of $L$ by construction, while $M$ is a constant matrix-valued multiplier placed between the
rank-one projections of GENERIC-FNO. This model satisfies the Jacobi identity and both degeneracy conditions exactly, but there
is nothing in it that ties $E_\theta$ to $Q$. It is integrated by explicit Euler, since its energy is not quadratic and there is
no exact linear step. Under exact degeneracy the energy error of this integrator is $O(\Delta t^2)$ per step, and the structural
residuals ($r_E\approx10^{-8}$) show that the integrator is not what distinguishes the ablation from PG-phys. \emph{PG-phys with
$\lambda\equiv0$} is the model of \S\ref{sec:method} without the Lie--Poisson term. It is Poisson at every state, and it cannot
represent nonlinear transport. Between them the two ablations bracket PG-phys: the first has exact structure without the
physics, the second has an exact Jacobi identity without Burgers.

\subsection{Learned Poisson structures by pullback}
\label{app:pullback}

A constant $L$ is a Darboux normal form, and any change of coordinates $\phi$ turns it into another Poisson operator,
$L_\phi=(D\phi)^{-1}L\,(D\phi)^{-\mathsf T}$, whose Casimirs are $C\circ\phi$ (Proposition~\ref{prop:pullback}). This is the function-space
analogue of Poisson neural networks \citep{jin2022poisson}, of which \S\ref{sec:model} is the special case $\phi=\mathrm{id}$. The
Lie--Poisson pencil arrives at state dependence by a different route, namely through compatibility and not through a change of
coordinates. A model based on pullback would need a map $\phi$ that is invertible and independent of the resolution, and its
Jacobian products would have to be paid for. We leave this to future work.

\section{Architecture, data and training protocol}
\label{app:protocol}
\label{sec:network_training}

\paragraph{Data.} On the periodic box $[0,2\pi)^d$ with $D=\sum_a\partial_a$, the equations are: heat $\partial_tu=\nu\Delta u$; advection
$\partial_tu=-c\,Du$ (transport along $(1,\dots,1)$ in 2D); viscous Burgers $\partial_tu=-u\,Du+\nu\Delta u$; damped wave $\partial_tq=v$,
$\partial_tv=c^2\Delta q-\gamma v$. Initial conditions are band-limited random fields ($|k|_\infty\le8$ in 1D, $\le6$ in 2D) with unit RMS.
We integrate them with a de-aliased spectral RK4 solver with resolution-dependent sub-stepping and save frames $\Delta t=0.1$
apart. Grids are $64$ points in 1D and $32^2$ in 2D, with 120 (1D) and 80 (2D) trajectories of 20 (15) frames, split 80/20. For
the damped wave we observe both $q$ and $v$.

\paragraph{Backbones.} A model uses a single backbone family throughout. The three families are an FNO (width 24, 12 modes in 1D and 8 in 2D, three
layers; InstanceNorm on in 2D and off in 1D, following the diagnosis of \citet{sulskis2026generic}), a Transolver (width 32, three
physics-attention blocks, four heads, 32 slices, with a kernel-3 circular depthwise convolution on the grid before each block),
and a CNO-style periodic convolutional operator (two-level U-shape, $3$-point circular kernels). The convolution in the Transolver
needs an explanation. It is the structured-mesh form of the Transolver's slicing, and without it every projection is a linear map
on flattened tokens, so that the only path between grid points goes through the 32 global slice tokens and the residual baseline
cannot represent a shift by one cell. In that case its training loss remains at the persistence level on advection, wave and
Burgers, and what the network learns is merely a pointwise damping. In the baseline the backbone is the residual map $G$; in the
PG models it is the entropy encoder $\psi_\eta$ and the coefficient network of $\Gamma_\vartheta$, while the pointwise networks
$\varphi_\theta$ and $e_\theta$ are small MLPs. The operator symbols
$\alpha,\lambda,d(k),b,b'$ are explicit parameters on the resolved band, with $m_L=16$ (1D) and $10$ (2D) for the reversible operator and
$m_M=N_x/3$ for the friction operator. Parameter counts are given in Table~\ref{tab:accuracy}.

\paragraph{Training.} All models minimize equation~\ref{eq:loss} with AdamW (weight decay $10^{-4}$), a one-cycle schedule, gradient
clipping at one and batch size 16 (1D) or 8 (2D), for 60 epochs in 1D and 40 in 2D, which amounts to about 350 optimizer steps
for every model. The learning rate is $10^{-3}$ for the networks and $3\times10^{-2}$ for the operator symbols. The reason for the
larger rate is that a symbol is a single scalar with a small gradient. For instance, the exact advection symbol is $-0.8$ at
$k=8$, and starting from $0.1$ it cannot be reached at the network rate in the number of optimizer steps available. The larger rate does
no harm, since the friction symbols are saturated by construction. PG-phys is trained and evaluated with the split integrator
($n_{\rm sub}=8$) and the learned-$E,S$ ablation with explicit Euler (Appendix~\ref{app:ablations}). The evaluation is a
ten-frame rollout from $t=0$ on held-out trajectories, and every number is a mean over three seeds, each with its own data and
initialization. The FNO and CNO grid comes from one fresh run of the released notebook on a single GPU; the Transolver cells
come from a second run of the same grid.

\paragraph{Training budget.} The FNO and CNO baselines converge within the shared protocol, but the Transolver baseline does not. With five times
as many epochs it reaches $0.0072\pm0.0009$ on 1D heat, $0.0064\pm0.0019$ on 1D advection, $0.070\pm0.004$ on 1D Burgers and $0.039\pm0.006$ on
the 1D wave (three seeds). The Transolver column of Table~\ref{tab:accuracy} should therefore be read as a comparison at equal
budget and not between converged models. Against a fully trained Transolver, PG-phys would lose on heat and still win on
advection, Burgers and the wave.

\paragraph{GENERIC-FNO.} The GENERIC-FNO columns of Table~\ref{tab:ladder} refer to the degeneracy-by-construction model of
\citet{sulskis2026generic}, which we ported from its public code. The functionals $E$ and $S$ are those of the original (a three-layer FNO of width $24$ and
$12$ modes with InstanceNorm, no coordinate channels, a $1\times1$ output, the spatial mean, and a $1\to16\to1$ head). The operators
are $L=(I-P_S)D_L(I-P_S)$ and $M=(I-P_E)D_M(I-P_E)$, where $D_L=i\,a(k)$ and $D_M=|b(k)|^2$ are diagonal Fourier multipliers on the lowest
$16$ modes per direction (two $k_x$ corners in 2D) and $P_w$ is the rank-one projection onto $w$. As in the original, the update
is a single explicit Euler step $u_{t+1}=u_t+L\,\delta E/\delta u+M\,\delta S/\delta u$ with the Euclidean autodiff gradient. The
$(q,v)$ wave was not treated in the GENERIC-FNO paper; for it we apply the multipliers per channel and let the projections act
on the whole state, and for $p=1$ the port reproduces the published model. We trained the port in two ways. In the first it has
the objective and protocol of all other models, and the rollout is differentiated through. In the second it has its own recipe,
which is a one-frame mean-squared error on a random frame pair per batch and, after epoch 30 and with probability $0.3$, a
three-frame rollout in which the input is detached before every frame; the optimizer is AdamW ($10^{-3}$, weight decay $10^{-4}$)
with a cosine schedule, batch $32$ and gradient clipping at one, for $150$ epochs in 1D and $120$ in 2D. When the weights are
shared, the port reproduces the $E$, $S$ and one-frame map of the original to $10^{-10}$ and $6\times10^{-8}$ (relative to the
increment) in double precision. Along the rollouts of Table~\ref{tab:ladder} its reversible degeneracy residual $r_E$ vanishes up
to round-off. Nevertheless its own $E$ and $S$ are not conserved and produced in discrete time, because an Euler step of the
size implied by its errors changes them at second order. This is precisely the point of the discrete-gradient step of
Appendix~\ref{app:integration}.

\paragraph{GENERIC-FNO on its own testbed.} In Table~\ref{tab:tmlr} the comparison is repeated on the data of the GENERIC-FNO
paper. We took over its generators (two to five random sinusoids of amplitude $0.3$--$0.5$), its frame spacing ($\Delta t=0.01$
for heat and advection, $0.005$ for Burgers), its protocol of $200$ samples and $20$ frames, and its metric, which is the mean
over $50$ test samples and $10$ frames of the per-sample relative error. Because the random amplitudes make this metric rather
noisy, we also give a batch-level relative $L^2$ error. Under its own recipe GENERIC-FNO ends up where its paper reports it
($0.20$, $0.36$, $0.033$ here with one seed, against published three-seed means of $0.095$, $0.23$, $0.046$). This is below
persistence on all three equations, and the physical energy grows by $1.8$--$2.7\times$ on heat and advection. On the same data
the unconstrained FNO of the present paper, a residual map trained on rollouts, is $35$--$400\times$ more accurate, and PG-phys
is $5$--$900\times$ more accurate. We conclude that the accuracy of GENERIC-FNO as published is the accuracy of its own FNO
baseline under its own one-frame recipe, and that the gap in Table~\ref{tab:ladder} cannot be blamed on our frame spacing. On
this testbed, with its small $\Delta t$, PG-phys is better than the FNO on advection and worse on heat and Burgers. On the latter
two the friction symbols have to shrink by two orders of magnitude from their initialization within the shared budget, and the
training loss is in fact still decreasing at the last epoch. The $\lambda\equiv0$ model comes closer.

\begin{table}[t]
\centering\small
\caption{On GENERIC-FNO's own testbed (its generators, $\Delta t$, 200-sample protocol; one seed, $64$ points). Each entry is GENERIC-FNO's
metric / batch-level relative $L^2$, both averaged over a ten-frame rollout. GENERIC-FNO uses its own recipe; the other models this
paper's protocol. Published GENERIC-FNO (its metric, three seeds): heat $0.095$, advection $0.233$, Burgers $0.046$.}
\label{tab:tmlr}
\footnotesize\setlength{\tabcolsep}{3pt}
\begin{tabular}{lccccc}
\toprule
& Persistence & Baseline (FNO) & GENERIC-FNO & PG-phys ($\lambda\equiv0$) & PG-phys \\
\midrule
heat & $0.053$ / $0.058$ & $0.0056$ / $0.0086$ & $0.204$ / $0.068$ & $0.037$ / $0.018$ & $0.044$ / $0.056$ \\
advection & $0.189$ / $0.198$ & $0.0009$ / $0.0012$ & $0.355$ / $0.207$ & $0.0003$ / $0.0008$ & $0.0004$ / $0.0011$ \\
Burgers & $0.018$ / $0.034$ & $0.0011$ / $0.0029$ & $0.033$ / $0.032$ & $0.0020$ / $0.015$ & $0.0026$ / $0.021$ \\
\midrule
$Q_{10}/Q^\star_{10}$, heat & -- & $1.00$ & $2.66$ & $0.96$ & $0.90$ \\
$Q_{10}/Q^\star_{10}$, advection & -- & $1.00$ & $1.82$ & $1.00$ & $1.00$ \\
$Q_{10}/Q^\star_{10}$, Burgers & -- & $1.00$ & $1.03$ & $1.00$ & $1.00$ \\
\bottomrule
\end{tabular}
\end{table}

\paragraph{Compute-matched baseline.} For each (equation, seed) the unconstrained FNO is trained again for the wall-clock time that
PG-phys needed on that cell. To this end a two-epoch probe measures the time per epoch of the baseline, the number of epochs is
set so as to fill the budget (with a cap of $6000$), and the model is trained from scratch with the one-cycle schedule stretched
over that number of epochs; it is then evaluated as in Table~\ref{tab:accuracy}. Some care is needed with the probe. It has to
follow an untimed warm-up epoch and to synchronize the device, since a cold probe over-estimates the epoch time by $40\times$.
On the same GPU PG-phys took $302$\,s for each 1D heat, advection and Burgers cell and $393$\,s for each wave
cell. This corresponds to $2180$--$2805$ baseline epochs, or $36$--$47\times$ the optimizer steps of the shared protocol (heat $36\times$,
advection $37\times$, Burgers $37\times$, wave $47\times$). The matched FNO reaches $0.0002$, $0.0008$, $0.030$, $0.0026$ at ten frames
(three seeds; Table~\ref{tab:ladder}). No matching was done for the 2D cells.

\paragraph{Long-horizon, data-efficiency and off-distribution tests.} Table~\ref{tab:relevance} is based on the trained 1D-FNO models of
seed 0. For the long-horizon test the solver is run for $100$ frames from the held-out initial conditions and each model is
rolled out over the same $100$ frames; the failure frame is defined as the first frame at which the relative $L^2$ error exceeds
one. For the data-efficiency test the baseline and PG-phys are trained from scratch on $25\%$ (and $50\%$) of the training
trajectories, with the protocol unchanged, and evaluated on the same held-out set. For the off-distribution tests we generate
new initial conditions, one set with twice the training amplitude and one with the band $|k|_\infty\le12$ in place of $8$,
integrate them with the same solver for $100$ frames, and evaluate every model on them without retraining. On these inputs
PG-phys is run with $16$ sub-steps, which is a choice made at inference time (Appendix~\ref{app:integration}). The
energy-increase count $n_\uparrow$ is the number of frames $t$ with $Q_{t}/Q_{t-1}>1+10^{-3}$, counted over the finite frames of
the rollout.

\paragraph{Jacobi residual on data states.} Table~\ref{tab:jacobi-data} evaluates $\mathcal J$ of the trained pencil on held-out
trajectories at frame $0$ and at the last frame of the training window. We use test functionals whose random fields lie on the
data band and functionals on the full grid, in 1D at $64$ points and in 2D at $32^2$.

\section{Mathematical Statements and Proofs}
\label{sec:app_proofs}

Throughout, $H=L^2(\Omega;\mathbb R^{p+1})$ with the grid-measure inner product $\ip{f}{g}=|\Omega|^{-1}\int_\Omega f\cdot g\,dx$, the state is $z=(u,s)$, brackets are $\{F,G\}=\ip{\delta F/\delta z}{L\,\delta G/\delta z}$ and $[F,G]=\ip{\delta F/\delta z}{M\,\delta G/\delta z}$, and every operator is asserted at a fixed but arbitrary state unless stated otherwise. On an $N$-point grid all statements hold verbatim with integrals replaced by grid sums; where the discrete case differs from the continuous one we say so.

\subsection{The reversible operator}

\begin{proposition}[Constant multipliers are Poisson]
\label{prop:poisson}
Let $L$ be a bounded skew-adjoint operator on $H$ with constant coefficients. Then $\{\cdot,\cdot\}$ is bilinear, skew, satisfies the Leibniz rule, and satisfies the Jacobi identity $\{\{F,G\},H\}+\{\{G,H\},F\}+\{\{H,F\},G\}=0$ for all $C^2$ functionals. In particular the block $\alpha D$ of equation~\ref{eq:pencil} on the band, and the wave block $J\,\mathrm{diag}(D_d,D_d)$ with $D_d$ symmetric, are Poisson.
\end{proposition}
\begin{proof}
Write $F_z=\delta F/\delta z$ and $F_{zz}$ for the second variation. By the chain rule, $\{\{F,G\},H\}=\ip{F_{zz}\,L G_z}{L H_z}+\ip{F_z}{L\,G_{zz}\,L H_z}$, the derivative of $L$ itself being zero. In the cyclic sum each functional's second variation appears twice, once as $\ip{F_{zz}LG_z}{LH_z}$ and once as $\ip{F_{zz}LH_z}{LG_z}$ with the opposite sign of the skew operator between them; symmetry of $F_{zz}$ and skewness of $L$ make the two cancel. The remaining terms cancel the same way. This is the standard argument for constant-coefficient Hamiltonian operators \citep[Cor.~7.5]{olver1993applications}. Skewness of the two blocks: $\alpha D$ has symbol $i\alpha k_D$, purely imaginary and odd in $k$; $J\,\mathrm{diag}(D_d,D_d)$ satisfies $(JD)^{\!\top}=D^{\!\top}J^{\!\top}=-DJ=-JD$ because $D=\mathrm{diag}(D_d,D_d)$ is symmetric and commutes with $J$.
\end{proof}

\begin{proposition}[The pencil]
\label{prop:pencil-app}
Let $D=\sum_a\partial_a$ on the torus and $P$ the $2/3$-rule projector. (i) $uD+Du$ is skew-adjoint, and so is $P(uD+Du)P$ on the grid. (ii) In the continuum, $\alpha D+\lambda(uD+Du)$ is a Poisson operator for all $(\alpha,\lambda)$. (iii) On the grid, the Jacobi identity for $\alpha D+\lambda P(uD+Du)P$ holds exactly on states and test functionals whose Fourier support is contained in the band on which cubic products are alias-free, and not in general otherwise.
\end{proposition}
\begin{proof}
(i) $\ip{f}{uDg+D(ug)}=\ip{uf}{Dg}-\ip{Df}{ug}=-\ip{D(uf)+uDf}{g}$ by skewness of $D$ and symmetry of pointwise multiplication; both properties survive discretization, and $P$ is symmetric, so the sandwich is skew on the grid. (ii) $D$ and $uD+Du$ are the Hamiltonian pair of the KdV hierarchy in the direction $D$; each is Hamiltonian and their sum is, which is the compatibility verified by \citet{magri1978simple}, see also \citet[\S7.3]{olver1993applications}; compatibility of two Hamiltonian operators is the Jacobi identity for every linear combination. (iii) For band-limited $z$ and test functionals with band-limited gradients, every quantity in the cyclic sum is a polynomial of degree at most three in band-limited fields; if the grid resolves those products without aliasing, the discrete computation coincides with the continuum one and the Jacobi identity holds by (ii). If the products alias, the discrete Leibniz rule $D(fg)=fDg+gDf$ fails on the aliased modes and the cancellation in (ii) does not go through; the residuals in Table~\ref{tab:structure-full} ($0.17$--$0.78$ on white-noise states, $10^{-16}$ on band-limited ones) are this effect.
\end{proof}

\subsection{Casimir entropy and the friction operator}

\begin{proposition}[The entropy is a Casimir]
\label{prop:casimir}
For $S[z]=|\Omega|^{-1}\int_\Omega s\,dx$ and any $L$ of the block form $\mathrm{diag}(L_u,0)$, $L\,\delta S/\delta z=0$ for every state.
\end{proposition}
\begin{proof}
$\delta S/\delta z=(0,1)$ and the $s$-block of $L$ is zero.
\end{proof}

\begin{proposition}[Symmetry, positivity, pointwise degeneracy]
\label{prop:M}
Let $A=D_M\,\mathsf P\,\Gamma_\vartheta(z)$ with $D_M$ as in equation~\ref{eq:DM}, $\Gamma_\vartheta$ any pointwise matrix field, and $\mathsf P(x)=I-\hat w\hat w^{\!\top}$ with $w=D_M^{\mathsf H}\,\delta E/\delta z$. Then $M=AA^{\!\top}$ is self-adjoint and positive semi-definite on real fields, and $M\,\delta E/\delta z=0$, for any parameters and any state.
\end{proposition}
\begin{proof}
$AA^{\!\top}$ is self-adjoint and $\ip{v}{AA^{\!\top}v}=\|A^{\!\top}v\|^2\ge0$ for any $A$; reality follows from the even, real symmetrization of the symbols in equation~\ref{eq:DM}. For degeneracy, $A^{\!\top}\,\delta E/\delta z=\Gamma_\vartheta^{\!\top}\mathsf P\,D_M^{\mathsf H}\,\delta E/\delta z=\Gamma_\vartheta^{\!\top}\mathsf Pw$, and $\mathsf Pw=w-\hat w\ip{\hat w}{w}_{\mathrm{slot}}=0$ pointwise. Hence $M\,\delta E/\delta z=A(A^{\!\top}\delta E/\delta z)=0$.
\end{proof}

\begin{proposition}[Onsager factorization and transport separation]
\label{prop:parity}
Under the flux/local slot and coefficient restrictions of equations~\ref{eq:DM} and \ref{eq:gamma}, the friction operator $M=AA^{\!\top}$ is self-adjoint and positive semi-definite for each frozen state. It therefore supplies no independently parameterized skew-adjoint transport generator. Its factorization accommodates flux-form diffusion and local damping, including the benchmark operators in Proposition~\ref{prop:represent}. This statement does not exclude first-order perturbation terms arising when state-dependent coefficients or the pointwise energy-flux projector are differentiated about a nonuniform state.
\end{proposition}
\begin{proof}
For any fixed state $z$, $\ip{f}{Mg}=\ip{A^{\!\top}f}{A^{\!\top}g}$, which is symmetric in $f,g$, and $\ip{f}{Mf}=\|A^{\!\top}f\|^2\ge0$. Thus $M$ has no skew-adjoint component at that state. The flux slots of $D_M$ and the local slots permit the diffusion and damping factorizations explicitly constructed in Proposition~\ref{prop:represent}. However, variation of $z\mapsto M(z)S_z$ differentiates $\Gamma_\vartheta(z)$ and $\mathsf P(z)$ and can produce first-order terms in the perturbation, even for variable-coefficient divergence-form diffusion. Consequently no stronger assertion that all terms of the full linearization have even differential order is made.
\end{proof}

\begin{proposition}[Representability]
\label{prop:represent}
\label{prop:repr}
On the resolved bands and constitutive ranges admitting the coefficient choices described below, the factorized operator of equations~\ref{eq:DM}--\ref{eq:friction} can be chosen so that $M\,\delta S/\delta z$ equals, exactly, $(\nu\Delta u,\ \nu|\nabla u|^2/T)$ for heat conduction and $(0,-\gamma v,\ \gamma v^2/T)$ for the damped wave, with $M\,\delta E/\delta z=0$ holding without the projector $\mathsf P$.
\end{proposition}
\begin{proof}
Heat, one channel $u$ and one channel $s$: take the flux symbol $-ik_a\kappa$ on the $u$ row of $D_M$ with a constant $\kappa$, identity on the local slot of $s$, and $\Gamma$ with $\Gamma_{(u,a),(u,a)}=\sqrt{\nu T}/\kappa$ and $\Gamma_{s,(u,a)}=-\sqrt\nu\,\partial_au/(\kappa\sqrt T)$, all other entries zero. Then $A^{\!\top}(f,g)=\Gamma^{\!\top}D_M^{\mathsf H}(f,g)=\sqrt\nu(\sqrt T\,\nabla f-g\,\nabla u/\sqrt T)$, which is $C^{\!\top}$ of equation~\ref{eq:heatM}, and $AA^{\!\top}(0,1)=(\nabla\!\cdot(\nu\nabla u),\ \nu|\nabla u|^2/T)$ by direct expansion. Degeneracy: $A^{\!\top}(u,T)=\sqrt\nu(\sqrt T\nabla u-T\nabla u/\sqrt T)=0$ pointwise, so $\mathsf P$ acts as the identity on the relevant directions and is not needed. Inverting the saturation in equation~\ref{eq:DM} gives $b_a(k)^2=\kappa^2/(1-\kappa^2k_a^2)$, finite exactly when $\kappa|k_a|<1$; choosing $\kappa^2=\nu\Delta t$ makes this the band $\nu\Delta t\,k^2<1$, which contains the data band ($\nu\Delta t\,k^2\le0.36$ there), and keeps $\Gamma$ within its bound $\|\Gamma-I\|\le1$ for $T\le4/(\nu\Delta t\,k_{\max}^2)\approx10$. Wave, channels $(q,v,s)$: take $D_M$ as the identity on the local slots of $v$ and $s$ (no flux slots used, $q$ row zero), $\Gamma_{v,v}=\sqrt{\gamma T}$, $\Gamma_{s,v}=-\sqrt\gamma\,v/\sqrt T$; then $A^{\!\top}(f_q,f_v,g)=\sqrt\gamma(\sqrt Tf_v-g\,v/\sqrt T)$, $AA^{\!\top}(0,0,1)=(0,-\gamma v,\gamma v^2/T)$, and $A^{\!\top}(-c^2\Delta q,v,T)=0$; no band condition arises because no flux slot is used. Both instantiations are checked numerically to $10^{-16}$ on band-limited states in \S\ref{sec:verification}.
\end{proof}

\subsection{The laws}

\begin{theorem}[Metriplectic structure and the two laws]
\label{thm:metriplectic}
\label{thm:laws}
For any parameters, $L$ of Definition~\ref{def:poisson} is skew-adjoint and Poisson (on every state if $\lambda=0$; on the resolved band otherwise), $M$ of equation~\ref{eq:friction} is self-adjoint and positive semi-definite, and $L\,\delta S/\delta z=0$, $M\,\delta E_\theta/\delta z=0$. Along $\partial_t z=L\,\delta E_\theta/\delta z+M\,\delta S/\delta z$,
\[
\frac{dE_\theta}{dt}=0,\qquad \frac{dS}{dt}=\ip{\delta S/\delta z}{M\,\delta S/\delta z}\ge0,
\]
and every Casimir of $L$---in particular $S$, and $\int u$ whenever $\lambda=0$---is invariant under the reversible flow $\partial_t z=L\,\delta E_\theta/\delta z$; for the pencil, $\int u$ is conserved by the reversible flow of $E_\theta$ (Proposition~\ref{prop:pencil}(iii)).
\end{theorem}
\begin{proof}
The structural statements are Propositions~\ref{prop:poisson}--\ref{prop:M}. Then $\dot E=\ip{E_z}{LE_z}+\ip{E_z}{MS_z}$; the first term is zero by skewness ($\ip{v}{Lv}=-\ip{Lv}{v}$), the second by $ME_z=0$ and self-adjointness of $M$. Likewise $\dot S=\ip{S_z}{LE_z}+\ip{S_z}{MS_z}=-\ip{LS_z}{E_z}+\ip{S_z}{MS_z}=0+\|A^{\!\top}S_z\|^2$. For a Casimir $C$, $\dot C=\ip{C_z}{LE_z}=-\ip{LC_z}{E_z}=0$ under the reversible flow; $\alpha D$ and the wave block vanish at $k=0$, so $\int u$ is a Casimir of $L$ when $\lambda=0$. For the pencil, $(uD+Du)\big(u+\varphi_\theta'(u)\big)=D\big(G(u)+u(u+\varphi_\theta'(u))\big)$ with $G'(t)=t\,(1+\varphi_\theta''(t))$ is a total derivative, hence of zero mean in the continuum and on alias-free grid states.
\end{proof}

\begin{theorem}[Bounded mechanical energy]
\label{thm:bound}
Write $\Pi_{\le2}\psi_\theta(u)=u^{\!\top}Au+b\cdot u+c_0$,
$e_\theta(s)=T_0s+\tilde e_\theta(s)$ with $\tilde e_\theta\ge0$, and
$B_0=E_\theta(z_0)-T_0 S(z_0)+c_0$.
For a scalar field with $Q=\frac12\|u\|^2$, if
$a=\lambda_{\max}(A)<\frac12$, then along the continuous flow and every
exactly solved discrete-gradient step,
\begin{equation}
(1-2a)Q(u_t)-|b|\sqrt{2Q(u_t)}\le B_0.
\label{eq:bound-scalar}
\end{equation}
In particular, with $h=1-2a>0$ and $x_t=\sqrt{2Q(u_t)}$,
\[
x_t\le\frac{|b|+\sqrt{|b|^2+2h B_0}}{h}
\]
whenever the square root is evaluated for a nonempty admissible initial
sublevel set. For the wave,
$Q(q,v)=\frac12(\|v\|^2+c_{\rm w}^2\|\nabla q\|^2)$,
the same conclusion holds if the mean of $q$ is fixed at zero (for example,
zero initial means of $q$ and $v$), the Poincar\'e inequality
$\|q\|^2\le C_P^2\|\nabla q\|^2$ applies, and
$2a_+K<1$, where $a_+=\max(a,0)$ and
$K=\max(1,C_P^2/c_{\rm w}^2)$. In that case,
\begin{equation}
(1-2a_+K)Q(q_t,v_t)-|b|\sqrt{2KQ(q_t,v_t)}\le B_0 .
\label{eq:bound-wave}
\end{equation}
Alternatively, a globally finite lower bound on $\varphi_\theta$
gives $Q(u_t)\le E_\theta(z_0)-T_0S(z_0)-\inf\varphi_\theta$
without the quadratic coercivity argument. The split integrator is not
covered by these exact discrete-time assertions.
\end{theorem}
\begin{proof}
Because $\psi_\theta\ge0$,
$\varphi_\theta(u)\ge-u^{\!\top}Au-b\cdot u-c_0$.
For scalar fields, Cauchy--Schwarz and the definition of $a$ yield
$\int\varphi_\theta(u)\ge -2aQ(u)-|b|\sqrt{2Q(u)}-c_0$,
where integrals use the normalized measure. Energy conservation,
entropy production, and $\tilde e_\theta\ge0$ imply
$Q+\int\varphi_\theta\le E_\theta(z_0)-T_0S(z_0)$,
giving equation~\ref{eq:bound-scalar}. Setting $x=\sqrt{2Q}$ gives
$\frac12 h x^2-|b|x-B_0\le0$, which bounds $x$ by its
larger root. For the mean-zero wave, Poincar\'e gives
$\|(q,v)\|^2\le2KQ$; thus
$\int\varphi_\theta\ge-2a_+KQ-|b|\sqrt{2KQ}-c_0$,
and the same conservation argument gives equation~\ref{eq:bound-wave}.
Its left side is coercive when $2a_+K<1$.
If $\inf\varphi_\theta>-\infty$ globally, use that lower bound
directly in the energy identity. The discrete-gradient case follows from
Theorem~\ref{thm:discrete} for an exactly solved implicit update.
\end{proof}

\begin{remark}
\label{rem:gauge}
The hypothesis is the quadratic gauge of \S\ref{sec:model}: a term $a|u|^2$ surviving in $\varphi_\theta$ is indistinguishable from rescaling $Q$, and $a\ge\tfrac12$ would flip the sign of the effective quadratic energy. It is also measurable. The exact linear step integrates the flow of $Q$ with the learned symbol, while the explicit remainder carries $2a$ of the same flow, so a trained model matches the data only if the learned symbol equals $(1+2a)^{-1}$ times the exact one. The symbol errors of \S\ref{sec:experiments} ($\le0.07\%$ on advection, up to $2\%$ on the wave, up to $2.3\%$ for Burgers $\lambda$) suggest a small residual quadratic gauge under this model interpretation, but do not independently certify the global curvature or the wave coercivity assumptions. For heat, which has no reversible symbol to compare, the bound column of Table~\ref{tab:calibration} uses $\inf\varphi_\theta$ over the data range and is loose ($\le0.36$).
\end{remark}

\begin{lemma}[Discrete gradients]
\label{lem:dg}
Let $\dg E(z,z')$ be the Gonzalez discrete gradient of equation~\ref{eq:dg} and $\dg S=(0,1)$. Then $\ip{\dg E}{z'-z}=E_\theta[z']-E_\theta[z]$, $\ip{\dg S}{z'-z}=S[z']-S[z]$, and $L(\bar z)\,\dg S=0$.
\end{lemma}
\begin{proof}
The first identity is the defining property of the Gonzalez discrete gradient \citep{gonzalez1996time,mclachlan1999geometric}; the second holds because $S$ is linear; the third because $\dg S=(0,1)$ lies in the $s$-block, which every $L$ of the form $\mathrm{diag}(L_u,0)$ annihilates, whatever the state at which $L_u$ is evaluated.
\end{proof}

\begin{theorem}[Exact discrete first and second laws]
\label{thm:discrete}
\label{thm:dg}
Let $z_{t+1}=z_t+\Delta t\bigl[L(\bar z)\,\dg E(z_t,z_{t+1})+\bar M\,\dg S\bigr]$ with $\bar M=\bar A\bar A^{\!\top}$, $\bar A=D_M\bar{\mathsf P}\,\Gamma_\vartheta(z^*)$ for any state $z^*$, and $\bar{\mathsf P}$ the projector off $D_M^{\mathsf H}\dg E$. Then, exactly and for any $\Delta t$,
\[
E_\theta[z_{t+1}]=E_\theta[z_t],\qquad S[z_{t+1}]-S[z_t]=\Delta t\,\|\bar A^{\!\top}\dg S\|^2\ \ge 0 .
\]
\end{theorem}
\begin{proof}
The argument is the discrete-gradient energy argument of \citet{gonzalez1996time} and its metriplectic extension \citep{romero2009thermodynamically,ottinger2018generic}, applied to the Casimir structure of Lemma~\ref{lem:dg}. By Lemma~\ref{lem:dg}, $E_\theta[z_{t+1}]-E_\theta[z_t]=\ip{\dg E}{z_{t+1}-z_t}=\Delta t\ip{\dg E}{L(\bar z)\dg E}+\Delta t\ip{\dg E}{\bar M\dg S}$. The first term vanishes by skewness of $L(\bar z)$ (Propositions~\ref{prop:poisson}, \ref{prop:pencil-app}(i)); the second because $\bar A^{\!\top}\dg E=\Gamma_\vartheta(z^*)^{\!\top}\bar{\mathsf P}\,D_M^{\mathsf H}\dg E=0$ by construction of $\bar{\mathsf P}$, whatever $z^*$, so $\bar M\dg E=0$, and $\bar M$ is self-adjoint. Likewise $S[z_{t+1}]-S[z_t]=\Delta t\ip{\dg S}{L(\bar z)\dg E}+\Delta t\ip{\dg S}{\bar M\dg S}$; the first term equals $-\Delta t\ip{L(\bar z)\dg S}{\dg E}=0$ by Lemma~\ref{lem:dg}, the second is $\Delta t\|\bar A^{\!\top}\dg S\|^2$.
\end{proof}

\begin{remark}
\label{rem:kernel}
The proof uses only that $\ker L(z)$ contains $\dg S$ for every $z$, which is the case because $S$ is a Casimir of every $L$ of block form, and that $\bar M\dg E=0$, which the pointwise projector provides for a coefficient field evaluated at any state. For the projected operator $(I-P_S(z))D_L(I-P_S(z))$ of GENERIC-FNO the kernel is spanned by $\delta S/\delta z|_z$ and moves with $z$; no single vector lies in $\ker L(z_t)$, $\ker L(\bar z)$ and $\ker L(z_{t+1})$ at once, and the first-order term in the entropy balance does not cancel. This is the concrete payoff of degeneracy through a Casimir rather than through a projection. The theorem concerns the solution of equation~\ref{eq:dgstep}; whether the fixed-point iteration of \S\ref{app:integration} reaches it is a separate question, answered there.
\end{remark}

\subsection{Two further statements}

\begin{proposition}[Pullback structures]
\label{prop:pullback}
For any $C^2$ diffeomorphism $\varphi$ of $H$ and any Poisson operator $L$, $L_\varphi=(D\varphi)^{-1}L\,(D\varphi)^{-\mathsf T}$ is Poisson, and $C$ is a Casimir of $L$ if and only if $C\circ\varphi$ is a Casimir of $L_\varphi$.
\end{proposition}
\begin{proof}
A Poisson operator is a contravariant 2-tensor and the Jacobi identity is coordinate-free, so the pullback of a Poisson tensor is Poisson \citep[Ch.~10]{marsden1999introduction}. By the chain rule $\delta(C\circ\varphi)/\delta z=(D\varphi)^{\!\top}\,\delta C/\delta w|_{\varphi(z)}$, hence $L_\varphi\,\delta(C\circ\varphi)/\delta z=(D\varphi)^{-1}L\,\delta C/\delta w$, which vanishes iff $L\,\delta C/\delta w=0$.
\end{proof}

\begin{proposition}[Resolution and dimension independence]
\label{prop:resolution}
Let $\iota_{N\to N'}$ be band-limited prolongation (Fourier zero-padding) and $\Phi_N$ one split or discrete-gradient step of the model on an $N$-point grid. For states whose Fourier support lies within the bands of $L$ and $M$ and below the de-aliasing limit of the coarser grid, $\Phi_{N'}\circ\iota_{N\to N'}=\iota_{N\to N'}\circ\Phi_N$.
\end{proposition}
\begin{proof}[Proof sketch]
Every component of the step is per-mode (the multipliers of $L$ and $D_M$, the exact unitary flow, the projector $P$), per-point (the constitutive functions, $\Gamma_\vartheta$, the slot projector $\mathsf P$), or a mean-pooled functional whose Riesz-scaled gradient is grid-independent by the inner-product convention of Appendix~\ref{app:innerproduct}. Per-mode maps commute with zero-padding on band-limited inputs; per-point maps commute with band-limited interpolation exactly when their outputs are resolved, which the band hypothesis guarantees; the encoder and the coefficient network are neural operators that are discretization-convergent by construction for the FNO backbone, and are excluded from the claim for the convolutional backbone and for the Transolver, whose local-mixing kernels are defined in grid units. The measured invariance to four digits in \S\ref{sec:integrator_study} is the empirical content.
\end{proof}

\section{Extended results, limitations and scope}
\label{app:extended}

\subsection{Structural identities: all equations and the learned-\texorpdfstring{$E,S$}{E,S} ablation}
\label{sec:app_structure}

Table~\ref{tab:structure-full} extends Table~\ref{tab:structure} to the wave, in one and two dimensions, and to the
learned-$E,S$ ablation. For both models the skewness of $L$, the symmetry of $M$, the two degeneracy conditions and
$\dot E=0$ hold to rounding error in every column, and the entropy production is positive. The two models differ in the
mass row. The learned-$E,S$ ablation has no conservation mask, so that its $M$ does not conserve $\int u$. The Jacobi rows
are discussed in \S\ref{sec:experiments}.

\begin{table}[!ht]
\centering\scriptsize
\setlength{\tabcolsep}{3pt}
\caption{Structural identities at random initialization, double precision, $16^d$ grid, maximum over a batch; Lrn is the learned-$E,S$ ablation, Phys is PG-phys. $\mathcal J$ is the normalized Jacobi residual of \S\ref{sec:diagnostics}; ``white-noise'' and ``band-limited'' refer to the test states and functionals. The Lie--Poisson rows are empty for the wave, whose operator has no pencil; the learned-$E,S$ ablation has no conservation mask, so its $M$ does not conserve $\int u$ (mass row); the projected-$L$ row uses a free learned entropy in both columns, since with a Casimir entropy the projection is trivial. Discrete-gradient entries are at the tolerance of the fixed-point solve, not the floating-point floor. Entries at or below the double-precision floor are written $\le\!10^{-16}$.}
\label{tab:structure-full}
\begin{tabularx}{\linewidth}{>{\raggedright\arraybackslash}X cccccccc}
\toprule
& \multicolumn{2}{c}{1D heat} & \multicolumn{2}{c}{2D Burgers} & \multicolumn{2}{c}{1D wave} & \multicolumn{2}{c}{2D wave} \\
\cmidrule(lr){2-3}\cmidrule(lr){4-5}\cmidrule(lr){6-7}\cmidrule(lr){8-9}
& Lrn & Phys & Lrn & Phys & Lrn & Phys & Lrn & Phys \\
\midrule
$\ip{v}{Lv}$ (skewness)
& $\le\!10^{-16}$ & $\le\!10^{-16}$ & $\le\!10^{-16}$ & $\le\!10^{-16}$ & $\le\!10^{-16}$ & $\le\!10^{-16}$ & $\le\!10^{-16}$ & $\le\!10^{-16}$ \\
$\ip{w}{Mv}-\ip{v}{Mw}$ (symmetry)
& $2\!\cdot\!10^{-16}$ & $\le\!10^{-16}$ & $\le\!10^{-16}$ & $\le\!10^{-16}$ & $1\!\cdot\!10^{-16}$ & $1\!\cdot\!10^{-16}$ & $\le\!10^{-16}$ & $\le\!10^{-16}$ \\
$\|L\,\delta S/\delta z\|/\|\delta S/\delta z\|$
& $4\!\cdot\!10^{-18}$ & $0$ & $8\!\cdot\!10^{-18}$ & $0$ & $2\!\cdot\!10^{-18}$ & $0$ & $8\!\cdot\!10^{-19}$ & $0$ \\
$\|M\,\delta E/\delta z\|/\|\delta E/\delta z\|$
& $0$ & $3\!\cdot\!10^{-18}$ & $0$ & $2\!\cdot\!10^{-18}$ & $0$ & $1\!\cdot\!10^{-19}$ & $0$ & $4\!\cdot\!10^{-20}$ \\
$|\sum_x(M\,\delta S/\delta z)_u|/\|M\,\delta S/\delta z\|$ (mass, $M$)
& $3.6$ & $\le\!10^{-16}$ & $14$ & $2\!\cdot\!10^{-16}$ & -- & -- & -- & -- \\
$|\sum_x(L\,\delta E/\delta z)_u|/\|L\,\delta E/\delta z\|$ (mass, pencil, data band)
& -- & $\le\!10^{-16}$ & -- & $8\!\cdot\!10^{-16}$ & -- & -- & -- & -- \\
$\ip{\delta E/\delta z}{\partial_t z}$ ($\dot E$)
& $1\!\cdot\!10^{-16}$ & $\le\!10^{-16}$ & $\le\!10^{-16}$ & $\le\!10^{-16}$ & $\le\!10^{-16}$ & $\le\!10^{-17}$ & $\le\!10^{-16}$ & $\le\!10^{-17}$ \\
$\min\ \ip{\delta S/\delta z}{\partial_t z}$ ($\dot S$)
& $0.58$ & $0.014$ & $0.68$ & $0.060$ & $0.34$ & $0.31$ & $0.78$ & $0.0027$ \\
\midrule
$\mathcal J$, our $L$, white-noise state
& $7\!\cdot\!10^{-16}$ & $1\!\cdot\!10^{-16}$ & $1\!\cdot\!10^{-16}$ & $3\!\cdot\!10^{-16}$ & $3\!\cdot\!10^{-16}$ & $4\!\cdot\!10^{-16}$ & $2\!\cdot\!10^{-16}$ & $2\!\cdot\!10^{-16}$ \\
$\mathcal J$, Lie--Poisson ($\lambda\!=\!0.3$), white-noise
& -- & $0.78$ & -- & $0.17$ & -- & -- & -- & -- \\
$\mathcal J$, Lie--Poisson ($\lambda\!=\!0.3$), band-limited
& -- & $3\!\cdot\!10^{-16}$ & -- & $4\!\cdot\!10^{-16}$ & -- & -- & -- & -- \\
$\mathcal J$, projected $L$ of GENERIC-FNO
& $1.0$ & $0.38$ & $5.0\!\cdot\!10^{-2}$ & $4.8\!\cdot\!10^{-2}$ & $1.4\!\cdot\!10^{-2}$ & $2.5\!\cdot\!10^{-2}$ & $7.6\!\cdot\!10^{-4}$ & $1.9\!\cdot\!10^{-3}$ \\
\midrule
$|E_1-E_0|/|E_0|$, discrete-gradient step
& $1\!\cdot\!10^{-16}$ & $1\!\cdot\!10^{-11}$ & $2\!\cdot\!10^{-16}$ & $8\!\cdot\!10^{-11}$ & $3\!\cdot\!10^{-16}$ & $2\!\cdot\!10^{-14}$ & $1\!\cdot\!10^{-16}$ & $5\!\cdot\!10^{-13}$ \\
$|E_1-E_0|/|E_0|$, explicit step
& $2\!\cdot\!10^{-11}$ & $5\!\cdot\!10^{-2}$ & $7\!\cdot\!10^{-11}$ & $5\!\cdot\!10^{-2}$ & $2\!\cdot\!10^{-11}$ & $5\!\cdot\!10^{-4}$ & $8\!\cdot\!10^{-12}$ & $3\!\cdot\!10^{-3}$ \\
\bottomrule
\end{tabularx}
\end{table}

\subsection{Dissipation calibration}

\begin{table}[t]
\centering\small
\caption{Dissipation calibration of PG-phys (three-seed means; FNO / Transolver / CNO): fractional one-frame decay of the physical energy against the truth, physical energy after ten frames relative to the truth, and $\max_tQ$ relative to the bound of Theorem~\ref{thm:bound}, evaluated using the reference quantity $a=0$ and $-\inf\varphi_\theta$ over the sampled data range (not a global theorem certificate; Remark~\ref{rem:gauge}). $\dagger$: Transolver entry over two seeds (the third diverged in training; \S\ref{sec:app_failures}).}
\label{tab:calibration}
\begin{tabular}{llccc}
\toprule
& & $\Pimodel/\Pistar$ & $Q_{10}/Q_{10}^\star$ & $\max_tQ/\text{bound}$ \\
\midrule
\multirow{4}{*}{1D} & heat & 0.92 / 1.00 / 1.00 & 1.03 / 1.00 / 1.00 & 0.26 / 0.14 / 0.13 \\
& Burgers & 1.01 / 1.01 / 1.01 & 1.00 / 0.99 / 1.00 & 0.16 / 0.15 / 0.10 \\
& wave$^\dagger$ & 0.77 / 0.90 / 0.76 & 1.07 / 1.03 / 1.08 & 0.89 / 0.76 / 0.89 \\
& advection & $|\Pimodel|\le10^{-4}$ & 1.000 / 1.000 / 1.000 & 0.47 / 0.42 / 0.45 \\
\midrule
\multirow{4}{*}{2D} & heat & 0.90 / 1.10 / 0.92 & 1.05 / 0.95 / 1.04 & 0.39 / 0.15 / 0.38 \\
& Burgers & 0.92 / 1.03 / 0.93 & 1.05 / 0.99 / 1.05 & 0.43 / 0.28 / 0.39 \\
& wave & 0.94 / 0.86 / 0.94 & 1.02 / 1.04 / 1.02 & 0.88 / 0.84 / 0.88 \\
& advection & $|\Pimodel|\le7\cdot10^{-4}$ & 0.998 / 0.998 / 0.996 & 0.45 / 0.47 / 0.43 \\
\bottomrule
\end{tabular}
\end{table}

PG-phys recovers $0.90$--$1.03$ of the true dissipation rate of heat and Burgers in eleven of the twelve cells, and $1.10$ in the
twelfth (2D Transolver heat). On the wave it recovers $0.76$--$0.94$. Advection is non-dissipative for all practical purposes
($|\Pimodel|\le7\times10^{-4}$), with $Q$ conserved to $4\times10^{-3}$ and $\rmech\le1.2\times10^{-3}$. The empirical, data-range ratio
$\max_tQ/\text{bound}$ is below one for all $71$ seeds that trained, and below $0.9$ in every cell. These numbers should be
compared with those of the learned-$E,S$ ablation, which satisfies the same degeneracy identities to machine precision while its
physical energy grows by up to $12\times$ on the wave and by $1.4$--$3.1\times$ on Burgers within ten frames. Exact degeneracy
constrains the functionals of the model. To constrain the physics, the energy has to be the physical one.

\subsection{Training failures}
\label{sec:app_failures}

Of the 72 seed-level runs, seventy-one remain finite and have a measured data-range bound ratio below one. The exception is one Transolver 1D-wave seed, which diverged in training (marked in
Table~\ref{tab:accuracy}). Among the ablations, the $\lambda\equiv0$ model diverged on three Burgers seeds (1D CNO, 1D Transolver, 2D FNO), that is, on the one equation it cannot represent.

\subsection{Long-horizon, data-efficiency and off-distribution tests}
\label{sec:app_relevance}

\begin{table}[t]
\centering\scriptsize
\setlength{\tabcolsep}{2.4pt}
\caption{Hundred-frame rollouts (1D, FNO backbone, seed 0; PG-phys is the seed-0 model of Table~\ref{tab:integrators}) in distribution, on inputs with twice the training amplitude, and on
inputs with band $|k|\le12$ (training: $\le8$); PG-phys uses $16$ sub-steps off distribution ($8$ in distribution, as trained). For
each: \emph{fail}, the first frame at which the rollout error exceeds one ($>100$ if never); $\epsilon$, the error at frame $100$;
$Q/Q^\star$, the physical energy relative to the solver's at frame $100$ (at the last finite frame if the rollout fails);
$n_\uparrow$, the number of frames on which $Q$ rises by more than $0.1\%$ (the solver: $0$ on every equation). \emph{25\% data}:
ten-frame error when retrained on a quarter of the trajectories (full data: Table~\ref{tab:ladder}).}
\label{tab:relevance}
\begin{tabular}{llcccc c cccc cccc}
\toprule
& & \multicolumn{4}{c}{in distribution} & 25\% data & \multicolumn{4}{c}{amplitude $\times2$} & \multicolumn{4}{c}{band $|k|\le12$} \\
\cmidrule(lr){3-6}\cmidrule(lr){7-7}\cmidrule(lr){8-11}\cmidrule(lr){12-15}
& & fail & $\epsilon$ & $Q/Q^\star$ & $n_\uparrow$ & $\epsilon_{10}$ & fail & $\epsilon$ & $Q/Q^\star$ & $n_\uparrow$ & fail & $\epsilon$ & $Q/Q^\star$ & $n_\uparrow$ \\
\midrule
\multirow{4}{*}{heat}
& Baseline & $>100$ & $0.015$ & $1.00$ & $0$ & $0.0076$ & $>100$ & $0.021$ & $1.00$ & $0$ & $>100$ & $0.107$ & $1.01$ & $0$ \\
& Matched & $>100$ & $0.0005$ & $1.00$ & $0$ & -- & $>100$ & $0.0029$ & $1.00$ & $0$ & $>100$ & $0.089$ & $1.01$ & $0$ \\
& GENERIC-FNO & $93$ & -- & $3.44$ & $0$ & -- & $77$ & -- & $3.79$ & $0$ & $94$ & -- & $3.46$ & $0$ \\
& PG-phys & $>100$ & $0.819$ & $2.19$ & $0$ & $0.034$ & $51$ & -- & $1.70$ & $1$ & $>100$ & $0.714$ & $2.08$ & $0$ \\
\midrule
\multirow{4}{*}{advection}
& Baseline & $>100$ & $0.045$ & $1.00$ & $0$ & $0.012$ & $>100$ & $0.352$ & $0.79$ & $0$ & $>100$ & $0.105$ & $0.98$ & $0$ \\
& Matched & $>100$ & $0.012$ & $1.00$ & $0$ & -- & $>100$ & $0.780$ & $0.74$ & $0$ & $>100$ & $0.094$ & $0.99$ & $0$ \\
& GENERIC-FNO & $15$ & -- & $3.44$ & $99$ & -- & $8$ & -- & $1.30$ & $63$ & $14$ & -- & $3.49$ & $99$ \\
& PG-phys & $>100$ & $0.031$ & $0.99$ & $0$ & $0.0085$ & $>100$ & $0.228$ & $1.00$ & $0$ & $>100$ & $0.035$ & $0.99$ & $0$ \\
\midrule
\multirow{4}{*}{Burgers}
& Baseline & $>100$ & $0.745$ & $2.53$ & $0$ & $0.171$ & $>100$ & $0.634$ & $2.25$ & $0$ & $>100$ & $0.754$ & $2.55$ & $0$ \\
& Matched & $>100$ & $0.090$ & $1.04$ & $0$ & -- & $>100$ & $0.206$ & $1.02$ & $0$ & $>100$ & $0.113$ & $1.05$ & $0$ \\
& GENERIC-FNO & $17$ & -- & $8.08$ & $7$ & -- & $6$ & -- & $97$ & $0$ & $20$ & -- & $7.00$ & $14$ \\
& PG-phys & $>100$ & $0.196$ & $1.09$ & $0$ & $0.083$ & $>100$ & $0.426$ & $1.03$ & $0$ & $>100$ & $0.195$ & $1.09$ & $0$ \\
\midrule
\multirow{4}{*}{wave}
& Baseline & $>100$ & $0.301$ & $0.89$ & $0$ & $0.402$ & $>100$ & $0.311$ & $0.76$ & $0$ & $>100$ & $0.407$ & $1.06$ & $0$ \\
& Matched & $>100$ & $0.017$ & $1.00$ & $0$ & -- & $>100$ & $0.046$ & $1.00$ & $0$ & $>100$ & $0.293$ & $0.82$ & $0$ \\
& GENERIC-FNO & $7$ & -- & $34.3$ & $87$ & -- & $7$ & -- & $25.4$ & $75$ & $8$ & -- & $32.2$ & $71$ \\
& PG-phys & $39$ & -- & $103$ & $3$ & $0.044$ & $8$ & -- & $1.18$ & $1$ & $33$ & -- & $10.2$ & $4$ \\
\bottomrule
\end{tabular}
\end{table}

The table is discussed in Section~\ref{sec:experiments}; here we add three remarks. The first concerns the column $n_\uparrow$,
which distinguishes two ways of getting the energy wrong. The unconstrained FNO never raises $Q$ on any input ($n_\uparrow=0$
throughout). It obeys no law in either direction, however: at equal optimizer steps it decays Burgers $2.5\times$ too slowly, and on the
$\times2$-amplitude advection inputs it loses $21$--$26\%$ of a conserved energy, even with $36\times$ the optimizer steps. GENERIC-FNO,
on the other hand, raises $Q$ on most frames of advection and of the wave. The second remark is that the heat ratio of $2.2$ for
PG-phys at frame $100$ comes from under-damping, which is monotone, and not from any creation of energy ($n_\uparrow=0$; the
ratio is $2.1$ on the wider band). Thirdly, the non-zero values of $n_\uparrow$ for PG-phys ($1$--$4$ on the wave rows, $1$ on
the $\times2$ heat row) all belong to the last frames before the split step diverges, and the entries $Q/Q^\star$ of those
rows are taken at the last finite frame of a diverging rollout. Up to those frames $Q$ is monotone. The wave divergence cannot
be removed by sub-cycling the split step. In an earlier training of the same seed the failure frame was $32$, $34$, $35$,
$38$ for $8$, $16$, $32$, $64$ sub-steps, while the ten-frame error did not change ($0.0416$--$0.0413$); the implicit step of
Appendix~\ref{sec:integrator_study} does remove it. For the $\times2$ Burgers inputs the situation is different. There
sub-cycling is sufficient for this seed, the rollout diverging with $8$ sub-steps and running to frame $100$ with $16$, although
it is not sufficient for all seeds (Table~\ref{tab:integrators}).

\subsection{Jacobi residual on data states}

\begin{table}[t]
\centering\small
\caption{Normalized Jacobi residual $\mathcal J$ of the trained pencil on held-out data states (Appendix~\ref{app:protocol}): at frame
$0$ and at the last frame of the training window, with test functionals on the data band and on the full grid. In 1D at $64$
points the data band ($|k|\le8$) keeps cubic products alias-free; in 2D at $32^2$ the data band ($|k|_\infty\le6$) does not. Data
states are stored in single precision, so the floor here is $10^{-8}$--$10^{-9}$ rather than the $10^{-16}$ of Table~\ref{tab:structure}.}
\label{tab:jacobi-data}
\begin{tabular}{llcccc}
\toprule
& & \multicolumn{2}{c}{functionals on the data band} & \multicolumn{2}{c}{functionals on the full grid} \\
\cmidrule(lr){3-4}\cmidrule(lr){5-6}
& & frame $0$ & last frame & frame $0$ & last frame \\
\midrule
\multirow{2}{*}{1D} & advection & $3.9\cdot10^{-8}$ & $5.2\cdot10^{-9}$ & $8.0\cdot10^{-2}$ & $1.5\cdot10^{-2}$ \\
& Burgers & $7.4\cdot10^{-9}$ & $7.4\cdot10^{-3}$ & $4.0\cdot10^{-2}$ & $9.2\cdot10^{-1}$ \\
\midrule
\multirow{2}{*}{2D} & advection & $9.0\cdot10^{-2}$ & $2.9\cdot10^{-1}$ & $7.2\cdot10^{-2}$ & $5.3\cdot10^{-2}$ \\
& Burgers & $9.0\cdot10^{-2}$ & $1.1\cdot10^{-1}$ & $7.2\cdot10^{-2}$ & $7.6\cdot10^{-2}$ \\
\bottomrule
\end{tabular}
\end{table}

With functionals on the data band, the 1D residual is at the single-precision floor at frame $0$ and $7\times10^{-3}$ at the end of
the Burgers window, by which time the solution has steepened. With functionals on the full grid it is $10^{-2}$--$0.9$, and in
2D, where the data band is not cubically resolved, it is $0.05$--$0.3$ for either class of functionals. We conclude that the
trained pencil is Poisson on the data band in the weak sense, that is, against test functionals which live on that band, and
that it is not Poisson on the aliased modes. The two laws, on the other hand, do not rely on the Jacobi identity. Both
$\dot E_\theta=0$ and $\dot S\ge0$ follow from the skewness of $L$ and the two degeneracy conditions, which hold at every state
(Table~\ref{tab:structure-full}). What aliasing affects is only whether the reversible flow on the aliased modes is
Hamiltonian. That flow is still skew, and since the de-aliasing projector keeps the output of the pencil on the resolved band,
thermodynamic consistency is not touched. Nor have we seen an instability that could be attributed to the residual, either in
the hundred-frame rollouts of Table~\ref{tab:relevance} or in the $\lambda\equiv0$ control, which satisfies the Jacobi identity
at every state. The $\lambda\equiv0$ model differs from PG-phys in what it can represent, not in its stability, and the wave,
on which PG-phys does diverge, has no pencil.

\subsection{Super-resolution and integration}
\label{sec:integrator_study}

The FNO version of PG-phys can be evaluated at other resolutions without modification. For three 1D-heat seeds the errors at $64/128/256$ points are
$0.01012/0.01011/0.01011$, $0.01806/0.01803/0.01802$ and $0.02011/0.02010/0.02010$, and the thermodynamic diagnostics agree to four
digits (Proposition~\ref{prop:resolution}). CNO is excluded from this test because its kernels are defined on the grid and are
not resolution-invariant; Transolver was not evaluated. The correct Riesz scaling of the variational derivative is essential
here. In an earlier version with Euclidean scaling the error grew from $0.011$ to $0.078$ to $0.125$ over the same three grids.
We also replaced the trained split step by
$3/6/12$ discrete-gradient iterations. This reduces the energy error per step from $1.2\times10^{-5}$ to
$3.5\times10^{-6}/1.9\times10^{-7}/2.6\times10^{-8}$, while the rollout error changes only from $0.0183$ to $0.0186$ and the thermodynamic
diagnostics do not change in their first two digits. The same study on a trained 2D-heat FNO model gives $9\times10^{-5}$ to
$10^{-6}/10^{-7}/1.5\times10^{-7}$, with the rollout error, $Q_{10}/Q^\star_{10}$ and $\max_tQ/\text{bound}$ unchanged to three digits. In
these studies the coefficient field is evaluated at the midpoint, and the fixed-point solve converges for single steps of the
FNO models. It does not converge on the trained 2D-heat Transolver models. Nor does it converge when it is used as the inference
integrator over a hundred frames on the trained 1D FNO models at $\Delta t=0.1$: the rollout leaves the finite range by frame
$4$--$18$ on heat, Burgers and the wave, and only on advection does it run, with the phase error of the midpoint rule that is
discussed below.

\paragraph{Newton--Krylov solve of the discrete-gradient step.}
The difficulty lies with the solver and not with the step. For Table~\ref{tab:integrators} and Figure~\ref{fig:integrators} we
solve equation~\ref{eq:dgstep} by an inexact Newton method. GMRES is applied to Jacobian--vector products of the residual, which
are obtained by reverse-over-reverse automatic differentiation so that the Jacobian is never formed, and the Newton update is
backtracked on the residual norm. If a solve does not reach a relative residual of $10^{-8}$, the step is repeated as two half
steps, down to $1/32$ of a frame if necessary. As in Theorem~\ref{thm:discrete}, $\Gamma_\vartheta$ is
evaluated at the midpoint inside the solve, and everything is done in double precision on the trained weights. To give an idea
of the difference: for one frame from trained states (seed 0), the residual after $6/12/50$ fixed-point iterations is
$4\times10^{-3}/1\times10^{-3}/2\times10^{-6}$ on Burgers and $0.026/0.028/0.057$ on the wave, where the iteration diverges, whereas Newton
reaches $4\times10^{-12}$ on both in three iterations and $24$--$35$ Krylov vectors. In the table the 1D FNO models of all three
seeds are rolled out for a hundred frames on $24$ held-out trajectories per input set, which makes $864$ trajectories per
integrator. Five observations can be made.

\emph{The laws.} With the solved step the energy of the model changes by at most $2\times10^{-10}$ of its value per frame in
distribution, and by at most $3\times10^{-9}$ on any rollout. For the split step, where it runs, the change is
$10^{-7}$--$4\times10^{-3}$ per frame. Moreover $S$ increases on every frame, and there is no frame on which the physical energy
rises, whereas there are $13$ such frames for the split step.

\emph{Divergence.} The split step loses $100$ of the $864$ trajectories. On the wave these are $7$ of $72$ in distribution (from
two of the three seeds, starting at frames $39$ and $29$), $11$ on the band-$12$ inputs and $63$ at twice the amplitude; the
remaining ones are $13$ on heat and $6$ on Burgers, both at twice the amplitude. The discrete-gradient step loses no trajectory
in distribution or on the wider band, on any equation, and none at all on heat or Burgers. On the $\times2$-amplitude wave it
loses $7$ of $72$ with four steps per frame and $9$ with one. A trajectory is stopped when the solve fails at the smallest step,
or when the latent entropy density has grown a thousand-fold. We also had to impose a wall-clock limit of $45$ minutes, which
three of the six $\times2$-amplitude wave rollouts of the solved step reached at frames $42$--$77$; these are counted over the
frames that were run.

\emph{Accuracy.} With one discrete-gradient step per frame the transport terms carry the phase error of the midpoint rule,
whereas the split step integrates them exactly. The error at ten frames is $0.044$ against $0.0023$ on advection, and $0.117$
against $0.053$ on the wave. With four steps per frame the phase error is gone ($0.0037$ and $0.054$; on Burgers $0.032$ against
$0.040$).

\emph{What the solver does not change.} Heat remains under-damped at late times whichever integrator is used (error $0.55$ at
frame $100$ on the seed mean; $Q/Q^\star=2.19$ for seed 0). On the wave the solved step brings to light the same miscalibration,
with the opposite sign. No trajectory is lost, but the error exceeds one at frame $47$ for one seed, and on the seed mean it is
$0.86$ at frame $100$ with four steps per frame. For seed 0 the model follows the energy of the solver up to frame $30$ ($0.43$
against $0.40$ of the initial energy); at frame $50$ it has half of it, and at frame $100$ only $1\%$.

\emph{Cost.} All costs are per frame, relative to the split step and on the seed mean. On heat, advection and Burgers the solve
costs $4$--$6\times$ with one step and $8$--$12\times$ with four (two to three Newton iterations per step). On the wave in
distribution it costs $24\times$ and $19\times$, and for seed 0, which needs steps of $1/8$ to $1/16$ of a frame, $59\times$ and
$35\times$. At twice the amplitude it costs two orders of magnitude. We also tried a cheaper alternative, an exponential
Rosenbrock--Euler step in place of the explicit sub-steps (matrix-free $\varphi_1$ in a twelve-dimensional Krylov space, one step
per frame). It cannot replace the solved step. At $1.3$--$1.6\times$ the cost of the split step it agrees with the split step in
distribution and removes the $\times2$-amplitude heat divergence for two of the three seeds; but it diverges by frame $5$ on the
$\times2$ Burgers inputs for all three seeds, loses $155$ trajectories altogether, and has no discrete law.

\emph{Other backbones and two dimensions.} The same solve was applied to the seed-0 models with a Transolver backbone in 1D
(heat, advection and Burgers; the wave seed did not train) and with FNO and Transolver backbones in 2D (all four equations), on
eight held-out trajectories each, over $100$ frames in 1D and $60$ in 2D. For a single frame Newton brings the residual below
$10^{-10}$ on all eleven models, in two to four iterations and with $9$--$38$ Krylov vectors. Fifty fixed-point iterations, by
contrast, diverge on the Transolver heat models in 1D and 2D and on the 2D Transolver Burgers model (residual $0.25$--$0.42$).
Along the rollouts the solved step changes $E_\theta$ by at most $4\times10^{-10}$ per frame, never raises $Q$ and loses no
trajectory, while the split step loses three of eight on 1D Transolver heat (from frame $58$) and four of eight on 2D Transolver
Burgers (from frame $26$). The cost is $4$--$7\times$ that of the split step with one step per frame and $8$--$12\times$ with
four, and this includes the 2D wave. As regards accuracy the picture is that of the 1D FNO models (2D advection at ten frames:
$0.004$ split, $0.10$ one step, $0.008$ four steps). The same is true of the model errors on which the solver has no influence.
In 2D the wave error exceeds one at frame $38$--$41$ under every integrator, without any trajectory being lost, and at frame
$60$ the model has a tenth of the solver's energy; the FNO Burgers model has $1\%$ of it.

\begin{table}[t]
\centering\scriptsize
\setlength{\tabcolsep}{2.2pt}
\caption{Integrators on the trained 1D FNO PG-phys models, three seeds: hundred-frame rollouts of $24$ held-out trajectories per
seed and input set. Split: the integrator of the experiments ($8$ explicit sub-steps in distribution, $16$ off it). Exp: the same
splitting with one exponential Rosenbrock--Euler step per frame. DG, DG/4: the discrete-gradient step of
Theorem~\ref{thm:discrete} solved by Newton--Krylov, one and four steps per frame. \emph{fail}: first frame at which the mean
error exceeds one or a trajectory is lost, per seed ($0/1/2$; a dot: none in $100$ frames; ${>}100$: none for any seed);
\emph{lost}: trajectories, of $72$, that became non-finite, grew a thousand-fold, or whose solve failed at $1/32$ of a frame;
$\epsilon_{10}$, $\epsilon_{100}$: seed-mean error at frames $10$ and $100$ (the latter if no trajectory is lost); $\delta E$:
largest $|\Delta E_\theta|/|E_\theta|$ over one frame; $n_\uparrow$: frames on which $Q$ rises by more than $0.1\%$, summed over
seeds; cost: time per frame relative to Split, in distribution.}
\label{tab:integrators}
\begin{tabular}{ll cccccc cccc cc c}
\toprule
& & \multicolumn{6}{c}{in distribution} & \multicolumn{4}{c}{amplitude $\times2$} & \multicolumn{2}{c}{band $|k|\le12$} & \\
\cmidrule(lr){3-8}\cmidrule(lr){9-12}\cmidrule(lr){13-14}
& & fail & lost & $\epsilon_{10}$ & $\epsilon_{100}$ & $\delta E$ & $n_\uparrow$ & fail & lost & $\delta E$ & $n_\uparrow$ & fail & lost & cost \\
\midrule
\multirow{4}{*}{heat} & Split & ${>}100$ & $0$ & $0.025$ & $0.547$ & $2\cdot10^{-4}$ & $0$ & $50/71/13$ & $13$ & $2\cdot10^{3}$ & $2$ & ${>}100$ & $0$ & $1.0$ \\
& Exp & ${>}100$ & $0$ & $0.026$ & $0.547$ & $3\cdot10^{-4}$ & $0$ & $\cdot/\cdot/12$ & $1$ & $3\cdot10^{4}$ & $1$ & ${>}100$ & $0$ & $1.6$ \\
& DG & ${>}100$ & $0$ & $0.026$ & $0.547$ & $3\cdot10^{-11}$ & $0$ & ${>}100$ & $0$ & $4\cdot10^{-11}$ & $0$ & ${>}100$ & $0$ & $3.6$ \\
& DG/4 & ${>}100$ & $0$ & $0.026$ & $0.547$ & $3\cdot10^{-11}$ & $0$ & ${>}100$ & $0$ & $1\cdot10^{-11}$ & $0$ & ${>}100$ & $0$ & $7.7$ \\
\midrule
\multirow{4}{*}{advection} & Split & ${>}100$ & $0$ & $0.0023$ & $0.022$ & $1\cdot10^{-7}$ & $0$ & ${>}100$ & $0$ & $9\cdot10^{-5}$ & $0$ & ${>}100$ & $0$ & $1.0$ \\
& Exp & ${>}100$ & $0$ & $0.0023$ & $0.022$ & $1\cdot10^{-7}$ & $0$ & ${>}100$ & $0$ & $9\cdot10^{-5}$ & $0$ & ${>}100$ & $0$ & $1.5$ \\
& DG & ${>}100$ & $0$ & $0.044$ & $0.328$ & $3\cdot10^{-12}$ & $0$ & ${>}100$ & $0$ & $6\cdot10^{-12}$ & $0$ & ${>}100$ & $0$ & $5.7$ \\
& DG/4 & ${>}100$ & $0$ & $0.0037$ & $0.036$ & $2\cdot10^{-10}$ & $0$ & ${>}100$ & $0$ & $3\cdot10^{-10}$ & $0$ & ${>}100$ & $0$ & $12$ \\
\midrule
\multirow{4}{*}{Burgers} & Split & ${>}100$ & $0$ & $0.040$ & $0.221$ & $4\cdot10^{-3}$ & $0$ & $\cdot/\cdot/16$ & $6$ & $5\cdot10^{0}$ & $0$ & ${>}100$ & $0$ & $1.0$ \\
& Exp & ${>}100$ & $0$ & $0.043$ & $0.225$ & $4\cdot10^{-3}$ & $0$ & $5/4/3$ & $61$ & $1\cdot10^{4}$ & $7$ & ${>}100$ & $0$ & $1.5$ \\
& DG & ${>}100$ & $0$ & $0.056$ & $0.224$ & $2\cdot10^{-11}$ & $0$ & ${>}100$ & $0$ & $4\cdot10^{-10}$ & $0$ & ${>}100$ & $0$ & $4.3$ \\
& DG/4 & ${>}100$ & $0$ & $0.032$ & $0.224$ & $4\cdot10^{-11}$ & $0$ & ${>}100$ & $0$ & $2\cdot10^{-11}$ & $0$ & ${>}100$ & $0$ & $9.3$ \\
\midrule
\multirow{4}{*}{wave} & Split & $39/\cdot/29$ & $7$ & $0.053$ & -- & $5\cdot10^{1}$ & $4$ & $8/10/9$ & $63$ & $1\cdot10^{3}$ & $2$ & $33/63/21$ & $11$ & $1.0$ \\
& Exp & $35/\cdot/26$ & $10$ & $0.053$ & -- & $2\cdot10^{2}$ & $3$ & $6/7/9$ & $67$ & $3\cdot10^{4}$ & $3$ & $27/63/20$ & $16$ & $1.3$ \\
& DG & $95/63/46$ & $0$ & $0.117$ & $1.04$ & $1\cdot10^{-10}$ & $0$ & $14/12/31$ & $9$ & $2\cdot10^{-9}$ & $0$ & $94/53/31$ & $0$ & $24$ \\
& DG/4 & $\cdot/\cdot/47$ & $0$ & $0.054$ & $0.858$ & $2\cdot10^{-10}$ & $0$ & $9/13/31$ & $7$ & $3\cdot10^{-9}$ & $0$ & $\cdot/63/31$ & $0$ & $19$ \\
\bottomrule
\end{tabular}
\end{table}

\begin{figure}[t]
\centering
\includegraphics[width=\linewidth]{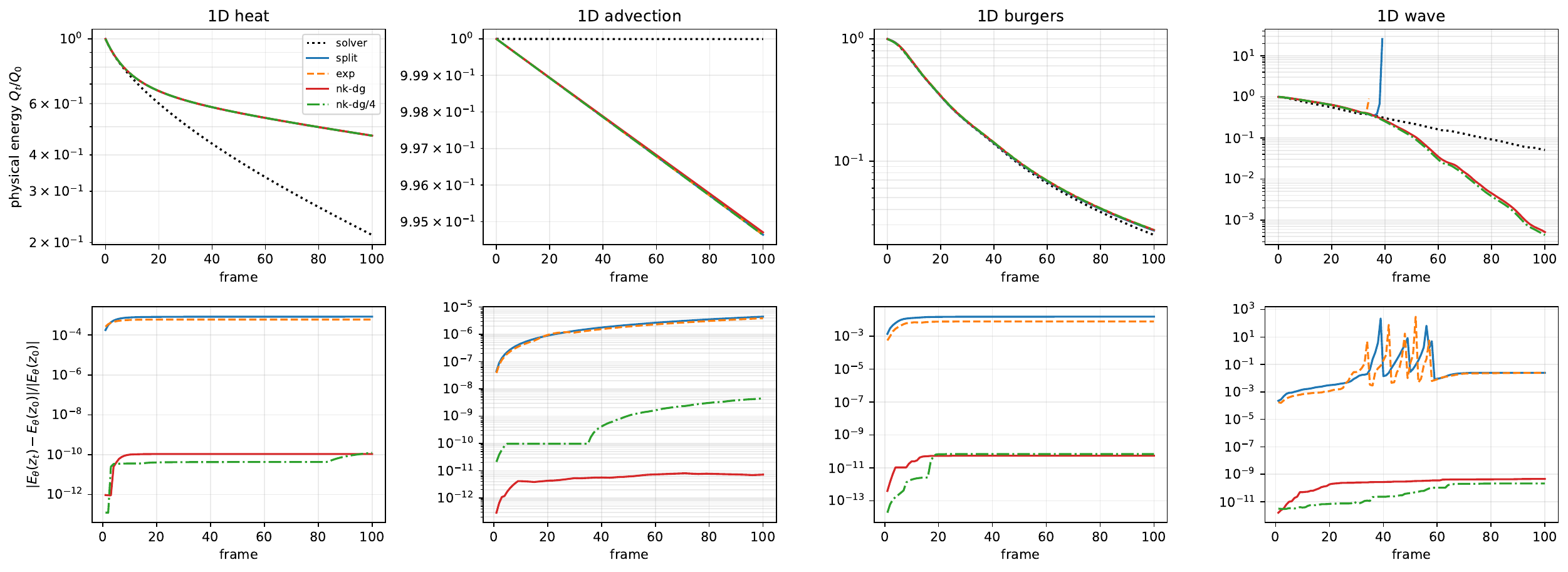}
\caption{The integrators of Table~\ref{tab:integrators} on the in-distribution rollouts of seed 0 (\texttt{nk-dg},
\texttt{nk-dg/4}: DG, DG/4). Top: physical energy $Q_t/Q_0$, mean over the $24$ trajectories (dotted: the solver). Bottom: drift
of the model's own energy, $|E_\theta(z_t)-E_\theta(z_0)|/|E_\theta(z_0)|$, worst trajectory. The split and exponential steps
drift by $10^{-6}$--$10^{-2}$ and diverge on the wave before frame $40$; the solved discrete-gradient step holds $E_\theta$ to
$10^{-8}$ on all four equations and keeps the wave rollout bounded, which is then over-damped.}
\label{fig:integrators}
\end{figure}

\subsection{Cost}

PG-phys has $0.3$--$0.8\times$ the parameters of the FNO/CNO baselines ($1.2$--$2\times$ for Transolver), but it trains $24$--$46\times$ more slowly (FNO,
CNO; $12$--$22\times$ for Transolver). The backbone is not responsible for this. The gradient of $Q$ is available in closed form ($u$, or $(-c^2\Delta q,v)$),
$e_\theta'$ and $\varphi_\theta'$ are derivatives of pointwise MLPs, and all operators act per mode or per point, so that every
component scales with the grid like an FFT. The cost has two sources, both of them choices of implementation. The coefficient
field $\Gamma_\vartheta(z)$ is evaluated again at each of the $n_{\rm sub}=8$ explicit sub-steps of a frame, and during training
we differentiate through all of them, which requires second-order automatic differentiation through the energy gradient.
Neither choice is forced by the structure. Since the laws hold for $\Gamma_\vartheta$ evaluated at any state
(Appendix~\ref{app:integration}), one evaluation per frame would be enough, and this alone saves seven of the eight evaluations
of the coefficient network. The sub-cycling, for its part, could be replaced by an implicit or exponential treatment of the
friction (both are tested at inference in Appendix~\ref{sec:integrator_study}). We estimate that an overhead of
$3$--$5\times$ would remain after both changes, most of it due to the second-order automatic differentiation of the pointwise
energy. Nothing in the construction changes in 3D, and the verification of GENERIC-FNO at $64^3$ used the same per-mode
operators; we have, however, not trained a 3D model.

\subsection{Limitations and scope}
\label{app:limitations}
\label{sec:limitations}

\paragraph{Jacobi identity.} With a constant $L$ the Jacobi identity holds at every state. Once the Lie--Poisson term is present
it holds to machine precision on alias-free states and functionals, but on white-noise states aliasing produces residuals of
$0.17$--$0.78$. On the data states of our experiments the residual is at the single-precision floor against functionals on the
data band; against functionals on the full grid it is $10^{-2}$--$0.9$ in 1D, and in 2D it is $0.05$--$0.3$
(Table~\ref{tab:jacobi-data}). In other words, the pencil is Poisson, in the weak sense, on the band that the data occupy. For a
problem whose solution fills the spectrum this would no longer be true, and one would need either a finer grid or a de-aliased
cubic product (a $1/2$-rule). The dispersive member $D^3$ of the pencil has been excluded for reasons of stability, and
dispersive equations are therefore outside the scope of this paper.

\paragraph{Cost of the pencil.} On the linear equations the model with a constant $L$ is the better one, by an order of magnitude on advection
and by a factor of $1.5$--$3$ on heat. In the Lie--Poisson model the transport coefficient and the learned potential must both
converge to zero for a gradient flow, and they must do so while the friction operator is still being learned. Heat is where this
costs most. The one cell that is lost on all three seeds is heat with the 1D FNO, and the 2D heat cells with FNO and CNO are
lost in the mean. On heat there is nothing for $L$ to do. The baseline can fit the heat propagator $e^{-\nu k^2\Delta t}$ as a
single spectral multiplier in one FNO layer, whereas our model must express it as
$D_M\Gamma\Gamma^{\!\top}D_M^{\mathsf H}$ applied to $(0,1)$, with $\Gamma$ depending on the temperature, i.e.\ as a product of learned
factors. Since the exact operator belongs to the family (Proposition~\ref{prop:represent}), the loss is one of optimization and
cannot be one of expressivity. Consistently with this, the $\lambda\equiv0$ model, which has less to learn, reduces the gap and
is ahead of PG-phys in every heat cell. On the small-$\Delta t$ testbed of GENERIC-FNO (Table~\ref{tab:tmlr}) the same loss
appears on Burgers too, because the friction symbols have to shrink by two orders of magnitude from their initialization. Two
tests suggest themselves: an initialization of the friction symbols that takes $\Delta t$ into account, and a sweep over
capacity and schedule for the friction operator alone.

\paragraph{Baselines.} All models are trained with one protocol of roughly 350 optimizer steps. The FNO and CNO baselines converge
within it. The Transolver baseline does not; it is still improving at the end, and with five times the budget it reaches
$0.007$ on 1D heat and $0.006$ on 1D advection (Appendix~\ref{app:protocol}). The Transolver column is therefore a comparison at
equal budget, and against a fully trained Transolver PG-phys would lose on heat and still win on advection, Burgers and the
wave. The $24$--$46\times$ training cost is controlled for by the compute-matched FNO of
Table~\ref{tab:ladder}. At equal wall-clock time this FNO is the more accurate model at ten
frames on all four 1D equations, by $1.4$--$80\times$ on heat, advection and the wave and by $10\%$ on Burgers, and it is still the
more accurate surrogate at a hundred frames in distribution (Table~\ref{tab:relevance}). The ten-frame advantage shown in
Table~\ref{tab:accuracy} is thus an advantage at equal optimizer steps. What the structure provides at any budget is the energy balance of
Section~\ref{sec:experiments}, and what one pays for it is accuracy per unit of wall-clock time.

\paragraph{Integration.} Under the linear flow the split step conserves the mechanical energy exactly, but it does not conserve
the total energy, and there is no a-priori bound on the energy. The consequences can be seen in
Tables~\ref{tab:relevance} and~\ref{tab:integrators}. In distribution the 1D wave rollout diverges for two of the three seeds
(at frames $29$ and $39$); on the $\times2$-amplitude heat and wave inputs the rollouts diverge for every seed; and in all these
cases the physical energy is monotone until the last few frames before the divergence. Sub-cycling can be chosen at inference
time and is sometimes sufficient ($16$ sub-steps hold the $\times2$ Burgers rollout for two of the three seeds). If exact
discrete laws are required, the implicit discrete-gradient step has to be used. With the fixed-point solve of the submitted
version this step is not a usable integrator on the trained models. With a Newton--Krylov solve it is
(Appendix~\ref{sec:integrator_study}), and it then removes every divergence in distribution and on the wider band, as well as
all divergences on heat and Burgers. Three limitations remain. The first is the cost, which per frame is
$4$--$12\times$ that of the split step on heat, advection and Burgers and $19$--$24\times$ on the wave on the seed mean
($35$--$59\times$ for the seed that needs steps of $1/8$ to $1/16$ of a frame). We tried two Fourier preconditioners for the
Krylov solve, the exact linear flow and the linearization at the mean state; the first halves the number of Krylov iterations
on advection, but neither changes the cost on the wave. The second limitation concerns the $\times2$-amplitude wave inputs, on
which the solve fails for $7$--$9$ of the $72$ trajectories at the smallest step we tried. Here Theorem~\ref{thm:bound} is of no
help, since it bounds the mechanical energy and says nothing about the regularity of the latent entropy density, which on these
inputs reaches three times its in-distribution range. The third is that the checks in 2D and with Transolver rest on one seed
and eight trajectories.

\paragraph{Late-time calibration.} The dissipation rate is calibrated over the ten-frame window of Table~\ref{tab:calibration},
in which the fast modes dominate. At a hundred frames on heat only the slowest modes are left, and for these PG-phys
under-damps: the physical energy is monotone, but it is $2.2\times$ that of the solver, and the error is $0.82$, against $0.0005$
for the compute-matched FNO (Table~\ref{tab:relevance}). The wave turns out to have the same defect with the opposite sign,
which becomes visible once its rollout is integrated with the laws exact. The model then follows the energy of the solver up to
frame $30$ and is over-damped afterwards, and the same is true of the 2D wave and Burgers models
(Appendix~\ref{sec:integrator_study}). An obvious remedy is to draw the four-frame training windows from hundred-frame
trajectories, but as it stands this does not work. On heat it does bring the ratio to $0.92$ and the error at frame $100$ to
$0.25$; the Burgers training then diverges, however, and the wave rollout fails earlier (frame $23$ instead of $39$). The
remaining candidate is a loss term on the late-time rate. None of this affects the laws.

\paragraph{Scope of the guarantees.} The guarantees refer to the functionals of the model itself. For any parameters $E_\theta$
is conserved and $S$ does not decrease, and the mechanical energy is bounded under the scalar or wave hypotheses of
Theorem~\ref{thm:bound}. The symbol comparisons are empirical diagnostics; they do not certify wave coercivity. Nor do the
guarantees imply that the learned flow agrees with the governing equation. This has to be assessed separately, by comparing
$\Pimodel$ with $\Pistar$ under the same normalization, and one should keep in mind that a model of this form applied to a forced
or open system, which is outside GENERIC, would still report residuals at machine precision. The present study is restricted to
scalar and $(q,v)$ fields on periodic domains, and to four canonical equations with known dissipation in one and two dimensions.
Extensions to three dimensions, to vector and multi-field systems, to non-periodic boundaries and to complex fluids, in which the
latent entropy density can represent a physical thermodynamic variable, are left for future work.

\end{document}